\documentclass{siamart190516}
\usepackage{graphicx,epsf,epstopdf}
\usepackage{amsmath,amsxtra,amsfonts,amscd,amssymb,bm}
\usepackage{algorithm}
\usepackage{algpseudocode}
\usepackage[top=2.5cm,bottom=2.5cm,right=2.5cm,left=2.5cm]{geometry}
\usepackage[mathscr]{eucal}
\usepackage{color}
\usepackage{subcaption}
\usepackage[export]{adjustbox}
\usepackage{booktabs}
\usepackage[normalem]{ulem}
\numberwithin{equation}{section}

\usepackage[nocompress]{cite}

\DeclareMathOperator*{\argmin}{arg\,min}

\def\eqnok#1{(\ref{#1})}

\newcommand{\edits}[1]{#1}
\newcommand{\revise}[1]{#1}
\newcommand{\nA}{{n_{\mathcal{A}}}}
\newcommand{\nS}{{n_{\mathcal{S}}}}

\newcommand{\tsum}{\textstyle\sum}
\newcommand{\bbe}{\mathbb{E}}
\def\prob{\mathop{\rm Prob}}

\newcommand{\beq}{\begin{equation}}
\newcommand{\eeq}{\end{equation}}
\newcommand{\beqa}{\begin{eqnarray}}
\newcommand{\eeqa}{\end{eqnarray}}
\newcommand{\beqas}{\begin{eqnarray*}}
\newcommand{\eeqas}{\end{eqnarray*}}

\newcommand{\bbr}{\Bbb{R}}
\newcommand{\nn}{\nonumber}

\def\cH{{\cal H}}
\def\cZ{{\cal Z}}
\def\cS{{\cal S}}

\def\cA{{\cal A}}
\def\cD{{\cal D}}
\def\cP{{\cal P}}

\def\KL{{\rm KL}}
\def\Pr{{\rm Pr}}
\def\cM{{\cal M}}
\def\dist{{D}}

\def\vgap{\vspace*{.1in}}

\title{Policy Optimization over
General State and Action Spaces
   \thanks{Authors are listed alphabetically. This research was partially supported by the AFOSR grant FA9550-22-1-0447 and the Department of Energy Computational Science Graduate Fellowship under Award Number DE-SC0022158.
The paper was first released at arXiv on 11/30/2022.
CJ was later included as an author for revising the paper and adding the policy evaluation scheme and numerical experiments.}}
\author{
     Caleb Ju\thanks{H. Milton Stewart School of Industrial and Systems
    Engineering, Georgia Institute of Technology, Atlanta, GA, 30332.
    (email: {\tt calebju4@gatech.edu}, {\tt george.lan@isye.gatech.edu})}
    \and
    Guanghui Lan\footnotemark[2]
 }

\date{November 29, 2022}

\begin{document}

\maketitle

\begin{abstract}

Reinforcement learning (RL)
problems over general state and action spaces are notoriously challenging.
In contrast to the tableau setting, one cannot enumerate all the states 
and then iteratively update the policies for each state. This prevents the application
of many well-studied RL methods, especially those with provable convergence
guarantees. In this paper, we first present
a substantial generalization of the recently developed policy mirror descent method to
deal with general state and action spaces. We introduce new approaches
to incorporate function approximation into this method so that we do not need to use explicit policy parameterization at all. 
Moreover, we present a novel policy dual averaging method for which possibly simpler function approximation
techniques can be applied. We establish a linear convergence rate to global optimality
or sublinear convergence to stationarity 
for these methods applied to solve different classes of RL problems under exact policy evaluation.
We then define proper notions of approximation errors for 
policy evaluation and investigate their impact on the convergence of these methods
applied to general-state RL problems with either finite-action or continuous-action spaces.
To the best of our knowledge, the development of these algorithmic frameworks
as well as their convergence analysis appear to be new in the literature.
Preliminary numerical results demonstrate the robustness of the aforementioned methods and show they can be competitive with state-of-the-art RL algorithms.
\end{abstract}



\section{Introduction} \label{sec_intro}
Stochastic dynamic programming provides a general framework to model
the interactions between an agent and its environment, and to improve the agent's decisions 
through such interactions. More specifically, the status of the environment is described by either discrete or continuous state
variables, while the agent's behavior is described by actions. Upon the agent's action, the system's state gets updated, and the agent
receives some reward or pays some cost.  The goal of stochastic dynamic programming is to find the
optimal policy which specifies the agent's best action at a given state. 

Markov decision process (MDP) is among the most widely used stochastic dynamic programming
models. Consider the infinite-horizon discounted Markov decision process
$M = (\cS, \cA, \cP, c, \gamma)$, where $\cS$
is the state space, $\cA$ denotes the action space,
$\cP: \cS \times \cS \times \cA \to \bbr$ is the transition 
model, $c: \cS \times \cA \to \bbr$ is the cost function,
and $\gamma \in [0,1)$ is the discount factor.
A policy $\pi: \cS \to \cA $ determines a particular action to be chosen at a given state. 
Since MDPs can be used to model many important sequential decision
making problems (e.g., inventory control, resource allocation, and organ donation and transplantation),
it has been a classic topic in the area of operations research and stochastic control~\cite{PutermanBook1994,BertsekasShreve1996}. 
While in the classic MDP literature, the transition model $\cP$
is assumed to be given when searching for the
optimal policy $\pi^*$, we often have no access to
exact information about $\cP$ in the more recent reinforcement learning (RL) literature~\cite{RichardBarto2018}.
This can be viewed as the main difference between MDP and RL.
During the past few years, RL has received widespread attention
partly due to its  success in computer games such as AlphaZero~\cite{alphazero1,alphazero2}.

In general, RL methods can be divided into model-based and model-free
approaches. In model-based methods, one seeks to 
approximate the transition model $\cP$ by using various statistical learning 
techniques and then to solve an approximate problem to 
the original MDP. In model-free methods, one
attempts to search for improved policies without
formulating an approximate transition model first.
The difference between model-based and model-free RL
is similar to the one between sample
average approximation (SAA) and stochastic approximation (SA)
in the area of stochastic optimization.
Model-based methods separate statistical analysis from
the design of optimization algorithms, while model-free methods
intend to carefully trade-off the effort needed for data collection and computation
through a combined design and analysis. Model-free methods
 appear to be more applicable to online reinforcement learning
where the policies need to be improved as streaming
data is being collected. These types of methods are the main subject of this paper.

In spite of the popularity of RL, the development of 
efficient RL methods
 with guaranteed convergence is still lacking in this area. 
Most existing theoretical studies focus only
on problems with finite state and finite action spaces.
In this setting, a rich set of model-free methods have been developed
by incorporating different sampling schemes into
the classic methods for MDPs,
including linear programming,
value iteration (or Q-learning), and policy iteration.
In addition, there has been considerable interest in the development of
nonlinear optimization-based  methods
that  utilize the gradient information (i.e., $Q$ function)
 for solving RL problems 
(e.g.,~\cite{SuttonMcAllester1999,kakade2001natural,KakadeLangford2002,10.2307/40538442,agarwal2021theory,abbasi2019politex,ShaniEM20,2020arXiv200706558C,Xu2020ImprovingSC,Wang2020NeuralPG,2020arXiv200506392M}).
While these methods are very popular in practice, 
it was only until recently that we knew these
stochastic gradient type methods can exhibit 
comparable and even superior performance to those based on classic dynamic
programming and/or linear programming approaches.
In particular, we show in \cite{LanPMD2021} that a general class of
policy mirror descent (PMD) methods can achieve linear convergence for solving
both regularized and unregularized MDP problems.
When applied to RL problems with unknown $\cP$,
it can achieve the optimal ${\cal O}(1/\epsilon^2)$ and ${\cal O}(1/\epsilon)$
sample complexity for solving unregularized and regularized problems, respectively.
It is worth noting the latter sample complexity has not been achieved even for
model-based approaches. Further interesting developments
of PMD-type methods can be found in
some more recent works (e.g.,~\cite{khodadadian2021linear,CayciHe2021,zhan2021policyPMD,xiao2022convergence,li2022homotopic,lan2022block,liTianjiao2022aMDP}). 

RL problems with general state and action spaces are notoriously 
more challenging than the finite-state finite-action setting mentioned above. 
One significant difficulty associated with a general state space
exists in that one cannot
enumerate all the states and properly define the policy found for each state.
In order to address this issue, the current practice is
to apply function approximation, e.g.,
by using a neural network, for policy parameterization (see, e.g., \cite{BhandariRusso2019,CayciHe2021}).
Then one can possibly apply stochastic 
gradient-type methods restricted to this policy class to search for the policy parameters.
However, the incorporation of policy parameterization
would destroy the possible convexity of the action space and make it
difficult to satisfy the constraints imposed on policies.
In many applications, meeting these constraints might be even more important
than to find an optimal policy, since the latter goal is usually not achievable
for problems with general state and action spaces due to their highly nonconvex landscape.
Moreover, with policy parameterization, it is often difficult to 
assess the quality of the solutions since 
the stationarity is defined with respect to (w.r.t.) the specific policy class rather than the original
policy optimization problem.
In addition, by using policy parameterization, we often end up with 
an optimization problem of very high dimension in the space of policy parameters,  
even if the dimension of the original action space is rather small.  
Finally, even with policy parameterization, another type of
function approximation is also needed to evaluate a policy and compute
the gradients, since one cannot save such gradient information in a tableau form.
In general, there are very few theoretical guarantees (e.g., feasibility, optimality, 
stationarity) on the convergence of existing RL methods
dealing with general state and action spaces.

This paper attempts to address some of the aforementioned issues
by presenting novel algorithmic frameworks to
solve RL problems over general state and action spaces.
The main contributions of this paper mainly consist of the following two aspects.
The first aspect of our contribution exists in a substantial generalization of the policy mirror descent (PMD) method.
More specifically, we start with an idealized situation 
for policy optimization, where exact gradient information can be computed over general state and action spaces.
We generalize the PMD method and show its linear convergence 
to the global optimality and sublinear convergence to a possible stationary point for solving
different classes of RL problems depending on the curvature of the action-value function.
We then discuss how function approximation can be incorporated into the PMD method.
In contrast to existing methods, we will not incorporate policy parameterization
into this algorithmic scheme at all. Instead, we apply function approximation only to the policy evaluation step 
applied to an augmented action-value function. The policies will be represented by a simple optimization problem, and will
be computed, whenever needed, based on
the parameters obtained during the policy evaluation step.  For RL problems with general state and finite action
spaces, this method converges to the global optimality  up to function approximation errors. Moreover,
its subproblems to compute the policy can have explicit solutions in this setting. Furthermore, for RL problems with general state and continuous
action spaces, it converges to the global optimality up to function approximation errors if 
a certain curvature condition is satisfied for the approximation functions. Otherwise,
it will converge to a possible stationary point for the original policy optimization problem under certain conditions for
the approximation error.

The second aspect of our contribution exists in the development of a novel class of
policy dual averaging (PDA) methods for solving RL problems.
The dual averaging method was originally developed for convex optimization~\cite{Nesterov2009}
and later further studied for machine learning (e.g.,~\cite{Xiao10-1}).
Similar to the original dual averaging method,
the PDA method developed herein minimizes a weighted summation of the 
advantage functions plus a certain regularization term at each iteration.
The PDA method is equivalent to PMD under the tabular setting with specific selection of distance generating functions,
but these two methods are different in general.
In comparison with PMD, the PDA method  appears to be more amenable to
function approximation as it only requires function approximation for the original action-value
function, and the stepsize parameters will not impact function approximation.
However, other than the tableau setting, the convergence of the PDA method has not been
studied before in the literature. Through a novel analysis, we show that
PDA also exhibits linear convergence to the global optimality and sublinear convergence to 
a possible stationary point (in a slightly
different sense than PMD) depending on the curvature of the action-value function.
We also study the impact of stochastic policy evaluation and function approximation
on the convergence of the PDA method for solving RL problems with general state and finite action 
spaces, as well as general state and continuous action spaces. In general,
PDA exhibits comparable convergence properties to PMD. In some cases it does have some advantages over PMD
by either allowing a more adaptive 
stepsize policy, e.g., when applied to general unregularized RL problems with finite action spaces, or requiring
less restrictive assumptions on the function approximation error, e.g., for solving RL problems with continuous action spaces.

This paper is organized as follows. We first formally introduce the problems of interest in~\cref{sec:s1}.
\Cref{sec:s2} then generalizes the PMD method to RL problems over general spaces.
The development of the new PDA method is detailed in~\cref{sec:exact_PDA}, whose structure is parallel to that of~\cref{sec:s2}.
Numerical experiments are studied in~\cref{sec:numerical_experiments}, and we end with some concluding remarks in~\cref{sec:conclusion}.

\subsection{Notation and Terminology} \label{sec:notations_PMD_General}
For simplicity, let the action space be $\cA \subseteq{\bbr^{\nA}}$ for some $\nA \geq 2$.
Let $\|\cdot\|$ be a given norm in $\cA$,
and let $\omega: \cA \to \bbr$ be a strongly convex function
w.r.t.~the given $\|\cdot\|$ which helps define Bregman's distance,
\beq \label{eq:omega_strong_convexity}
D(a_1, a_2) := \omega(a_2) - 
[
\omega(a_1) + \langle \omega'(a_1), a_2 - a_1 \rangle 
] \ge \tfrac{1}{2} \|a_1 - a_2\|^2, \forall a_1, a_2 \in \cA.
\eeq
For general-state and finite-action MDPs (see \eqnok{eq:def_simplex}-\eqnok{eq:dis_kernel}),
a common selection of $\omega$ for RL is the entropy function $\omega(a) = \tsum_{i=1}^{\nA} a_i \log a_i$.
With this selection, the Bregman distance becomes the KL divergence (see \cite{LanPMD2021}) given by
$
D(a, a') = \KL(a' \parallel a) = \tsum_{i=1}^{\nA} a'_i \log \tfrac{a'_i}{a_i}.
$
For general-state and continuous-action MDPs, the selection of $\omega$ will be problem-dependent.
A common selection would be the Euclidean norm $\omega(a) = \tsum_{i=1}^{\nA} a_i^2/2$, with the corresponding
Bregman distance given by 
$
D(a', a) = \tfrac{1}{2}\|a - a'\|_2^2.
$

\section{Problem of Interest} \label{sec:s1}
For a space $\Omega$ with a $\sigma$-algebra $\sigma_\Omega$,
we define $\cM(\Omega)$ as the set of 
all probability measures over $\sigma_\Omega$.
\revise{Following~\cite{hernandez2012discrete}, we endow a Borel space (i.e.,~a Borel subset of a complete and separable metric space) with its Borel $\sigma$-algebra. 
In this paper, the term ``measurable'' refers to a Borel measurable set or function depending on the context.}

\subsection{Markov Decision Processes} \label{sec_prob_formulation}
An infinite-horizon discounted Markov decision process (MDP) is a five-tuple
$(\cS, \cA, \cP, c, \gamma)$, where 
$\cS$ is a measurable 
state space, $\cA \subseteq \bbr^\nA$
is a closed convex set, 
$\cP: \cS \times \cA \to \cM(\cS)$
is a mapping with domain $\cS \times \cA$,
$c: \cS \times \cA \to \bbr$ denotes the cost function,
and $\gamma \in [0,1)$ is a discount factor.
The transition  kernel $\cP$ is a stochastic kernel~\cite[Definition C.1]{hernandez2012discrete}.
The measurable cost function $c(s, a)$
specifies the instantaneous cost associated with
the selection of action $a$ in state $s$.
\revise{While in general RL problems, both the state space $\cS$ and action space $\cA$ can be any (Borel) measurable space~\cite{hernandez2012discrete}, in this paper we require $\cA \subseteq \mathbb{R}^{\nA}$ to be a closed convex set. 
This is for a couple key reasons.
First, restricting $\cA$ to be a closed convex set allows us to design efficient algorithms for optimizing over the action space.
Second, by working with general, measurable state spaces $\cS$, we can provide a generic framework that covers both finite and infinite state space problems, as well as a Cartesian product of spaces with varying sizes, i.e.,~mixed discrete and continuous spaces. Such settings arise in queuing, partially observable models, PDE-based control, and periodical problems~\cite{PutermanBook1994,hernandez2012discrete,farahmand2017deep,shapiro2020periodical}.}

A deterministic policy $\pi: \cS \to \cA$ 
is a measurable mapping
that determines a single feasible action $\pi(s) \in \cA$
for each state $s \in \cS$.
In particular, our developments cover the finite set of actions $\{\cA_1, \ldots, \cA_{\nA}\}$, as an important special case. 
Indeed, when the set of possible actions is finite, a widely used approach in policy optimization
is to seek a randomized policy that determines the probability of selecting a specific action.
This is equivalent to specifying
a continuous action $a$ over the simplex set
\beq \label{eq:def_simplex}
\cA = \Delta_{\nA} := \{p \in \bbr^{\nA}: \tsum_{i=1}^{\nA} p_i = 1, p_i \ge 0, i = 1, \ldots, \nA\},
\eeq
with each of its extreme points corresponding to
an action $\cA_i$, $i = 1, \ldots, \nA$. Clearly, $\cA \subseteq \mathbb{R}^{n_\cA}$. 
Accordingly, the cost function and the transition kernel,  respectively, are given by 
\beq \label{eq:dis_kernel}
c(s,a) = \tsum_{i=1}^{\nA} a_i c_0(s, \cA_i) \ \ \mbox{and}
\ \ \cP(s,a) = \tsum_{i=1}^{\nA} a_i \cP_0(\cdot| s, \cA_i),
\eeq
Here $c_0(s, \cA_i)$ and $\cP_0(\cdot| s, \cA_i)$
denote the cost and transition probability when
the action $\cA_i$ is chosen.
In the rest of the paper, we will refer to the aforementioned situation as
general-state and finite-action MDPs even though the 
action space is given by the simplex as in \eqnok{eq:def_simplex}.
More generally, our development also covers 
an action space with mixed discrete and continuous
actions by employing randomization over
the discrete ones.

Let $\Pi$ denote the set of all measurable 
stationary policies that are invariant to the time epoch $t$.
For any $\pi \in \Pi$, 
we measure its performance by the action-value
function $Q^\pi: \cS \times \cA \to \bbr$
defined as
\[
Q^\pi(s,a) :=
\bbe\left[\tsum_{t=0}^\infty \gamma^t [c(s_t, a_t) + h^{a_t}(s_t)] \mid
s_0 = s, a_0 = a, a_t = \pi(s_t),
s_{t+1} \sim \cP(\cdot | s_t, a_t)\right],
\]
where $h^a(s)$ is a
regularizer which is strongly convex  w.r.t. $a$ with modulus $\mu_h \ge 0$, i.e.,
\[
h^{a_1}(s) - h^{a_2}(s) - \langle ( h^{a_2})'(s), a_1 - a_2 \rangle \ge
\mu_h D(a_2, a_1), 
\]
\revise{where $(h^{a_2})'(s)$ denotes a subgradient of $h^a(s)$ at $a=a_2$.}
In particular, entropy regularization has been shown to incentivize safe exploration and to learn risk-sensitive policies~\cite{neu2017unified,geist2019theory,ahmed2019understanding}. 
Our regularization is more general, e.g., it can model penalty for constraints.
Here we separate $h^{a_t}(s_t)$ from $c(s_t, a_t)$ in order to take advantage of
its strong convexity in the design and analysis of algorithms.
Moreover, we define the state-value function
$V^\pi: \cS \to \bbr$ associated with $\pi$
as
\beq \label{eq:def_V_function}
V^\pi(s) :=
\bbe
 \left[
 \tsum_{t=0}^\infty \gamma^t [c(s_t, a_t) + h^{a_t}(s_t)] \mid
s_0 = s, a_t = \pi(s_t),
s_{t+1} \sim \cP(\cdot | s_t, a_t)
\right].
\eeq 
Throughout the paper, we assume that
$c(s, a) + h^a(s) \in [0, \bar c]$ so that
$V^\pi(s) \in [0, \tfrac{\bar c}{1-\gamma}]$.
It can be easily seen from the definitions of $Q^\pi$ and $V^\pi$ that
\begin{align}
V^\pi(s) &= Q^\pi(s, \pi(s)), \label{eq:QV1} \\
Q^\pi(s, a) &= c(s, a) +h^{a}(s) + \gamma \int  \cP(d s'| s, a) V^\pi(s'). \label{eq:QV2}
\end{align}
Without specific mention, we assume throughout the paper that the integral is taken over the state space $\cS$ and is well-defined. 
The latter is guaranteed since all corresponding integrands appearing in the paper are assumed to be finite.

The main objective in MDP is to find an optimal policy $\pi^*: \cS \to \cA$ s.t.
\beq \label{eq:opt_objective}
V^{\pi^*}(s) \le V^\pi(s), \forall \pi \in \Pi, \forall s \in \cS.
\eeq 
Sufficient conditions that guarantee the existence of $\pi^*$
have been intensively studied (e.g., \cite{BertsekasShreve1996,PutermanBook1994,hernandez2012discrete}).
For example, one such sufficient condition is given by
the existence of a measurable policy $\pi \in \Pi$
such that
\[
c(s, \pi(s)) + h^{\pi(s)}(s) + \gamma \int V(s') \cP(d s'| s, \pi(s))
= \inf_{a \in \cA} \left [c(s, a) + h^a(s) + \gamma \int V(s') \cP(ds'|s, a)\right].
\] 
The existence of such a $\pi$ is guaranteed under the so-called the \textit{measurable selection condition}~\cite{hernandez2012discrete}, \revise{which provides sufficient conditions on $\cA$, $c$, and $\cP$ such that the equality above holds across all states $s \in \cS$.}
Throughout the paper, we assume the measurable selection condition holds.
Under these conditions,  
\eqnok{eq:opt_objective} can be formulated
as a nonlinear optimization 
problem with a single objective function.
More specifically, 
given an initial state distribution
$\rho$ over $\cS$, let $f_\rho$ be defined as the functional
\beq \label{eq:RL0}
f_\rho(\pi)
:= \int V^\pi(s) \rho(ds).
\eeq
Our goal is to solve the following policy
optimization problem
$
 \min_{\pi(s) \in \cA, \forall s \in \cS} f_\rho(\pi).
$

For a given policy $\pi$,
let $\cP^\pi$ be its associated transition kernel, i.e.,
$\cP^\pi(\cdot| s) = \cP(\cdot| s, \pi(s))$.
A probability distribution
$\nu^\pi$
is said to be a stationary distribution of
$\cP^\pi$ if
for all $B \in \sigma_\cS$,
$\nu^\pi(B)
= \int \cP^\pi(B|s) \nu^\pi(ds)$.
Existence and uniqueness of a stationary
distribution $\nu^\pi$ and
convergence of a Harris chain (i.e.,~Markov chains over general state spaces) to $\nu^\pi$
is ensured by ergodicity conditions.
A sufficient condition is that the chain is \revise{an ergodic Harris chain, which means it is positive recurrent (i.e.,~has a regeneration set $R$ whose regeneration time has finite expected value) and the embedded renewal process is aperiodic.
We refer to~\cite[Chapter 7]{bladt2017matrix} for a more detailed presentation.}
For discrete-state Markov chains, 
it is well-known that any aperiodic
and positive recurrent Markov chain has a unique
invariant stationary distribution.
Throughout this paper, we assume that
$\nu^\pi$ is the unique stationary
distribution associated with $\cP^\pi$.

In particular, we use $\nu^* := \nu^{\pi^*}$
to denote the stationary distribution 
induced by the optimal policy $\pi^*$.
To simplify our discussion, 
we choose $\rho$ in \eqnok{eq:RL0}
as the stationary distribution $\nu^*$.
While the distribution $\rho$ can be arbitrarily chosen, it has been recently observed (e.g.,~\cite{LanPMD2021}) that the choice of $\nu^*$ can simplify the analysis of RL algorithms.
In this case, our problem of interest becomes
\beq \label{eq:MDPNLP}
f^* := \min_{\pi(s) \in \cA, \forall s \in \cS} \{f(\pi) := f_{\nu^*}(\pi) \}.
\eeq
It is clear that 
$\pi^*$ must be an optimal solution of problem~\eqnok{eq:MDPNLP}.
Moreover, an optimal solution of \eqnok{eq:MDPNLP}
must be an optimal policy for our MDP if $\nu^*$ is sufficiently random
(i.e., positively supported over $\cS$).

\subsection{Performance Difference and Policy Gradient} \label{sec:perf_diff_and_policy_grad}
Given a feasible policy $\pi \in \Pi$, we
define the discounted state visitation measure by
\beq \label{eq:def_visit_measure}
\kappa_s^\pi(B) = (1-\gamma) \tsum_{t=0}^\infty \gamma^t \Pr^\pi(s_t \in B| s_0 = s)
\eeq
for any measurable set $B \subset \cS$.
Here, $\Pr^\pi(s_t \in \cdot | s_0=s)$ denotes
the distribution of the state at time $s_t$ after we follow the policy $\pi$
starting from state $s_0 = s$. 
Since $\mathrm{Pr}^\pi(s_t \in B \vert s_0=s) \in [0,+\infty)$ and $\gamma <1$, the sum in $\kappa_s^\pi(B)$ is absolutely convergent and hence is well-defined.
Note that $\kappa_s^\pi$ is a probability distribution over states and also satisfies a set of balance equations~\cite[Lemma 17]{BhandariRusso2019}. 
Let $(\cP^\pi)^t$ be
the $t$-step transition kernel
associated with $\cP^\pi$ defined recursively
according to
\[
(\cP^\pi)^t(B|s) = \int (\cP^\pi)^{t-1}( ds' |s) \cP^\pi( B |s'), ~t \ge 2.
\]
Then we have
$
\Pr^\pi(s_t \in B| s_0=s) = (\cP^\pi)^t(B|s).
$

We now state an important ``performance difference'' lemma
which tells us the difference on the value functions
for two different policies.

\begin{lemma} \label{lem:performance_diff_deter}
Let $\pi$ and $\pi'$ be two feasible policies. Then
we have
\begin{align}
V^{\pi'}(s) - V^\pi(s) 
&= \tfrac{1}{1-\gamma} \int \psi^\pi(q, \pi'(q))  \, \kappa_s^{\pi'} (dq), \forall s \in \cS, \label{eq:per_diff}
\end{align}
where 
\beq \label{eq:def_advantage}
\psi^\pi(q, a) := Q^\pi(q, a) - V^\pi(q) +  h^{a}(q) - h^{\pi(q)}(q).
\eeq
\end{lemma}

\begin{proof}
By the definition of $V^{\pi'}$ in \eqnok{eq:def_V_function}, we have
\begin{align*}
&V^{\pi'}(s) - V^\pi(s) \\
&=
\bbe
 \left[
 \tsum_{t=0}^\infty \gamma^t [c(s_t, a_t) + h^{\pi'(s_t)}(s_t)]\mid
s_0 = s, a_t = \pi'(s_t),
s_{t+1} \sim \cP(\cdot | s_t, a_t)
\right] - V^\pi(s) \\
&= \bbe
 \left[
 \tsum_{t=0}^\infty \gamma^t [c(s_t, a_t) + h^{\pi'(s_t)}(s_t)+\gamma V^\pi(s_{t+1})
 - V^\pi(s_t)] \right. \\
&\quad \quad \quad \left.  \mid s_0 = s, a_t = \pi'(s_t),
s_{t+1} \sim \cP(\cdot | s_t, a_t)
\right] \\
&\quad + \bbe\left[V^\pi(s_0)\mid
s_0 = s, a_t = \pi'(s_t),
s_{t+1} \sim \cP(\cdot | s_t, a_t)
\right]  - V^\pi(s) \\
&= \bbe
 \left[
 \tsum_{t=0}^\infty \gamma^t [c(s_t, a_t) + h^{\pi(s_t)}(s_t) + \gamma V^\pi(s_{t+1})
 - V^\pi(s_t) +  h^{\pi'(s_t)}(s_t) - h^{\pi(s_t)}(s_t)] \right.\\
&\quad \quad \quad \left. \mid
s_0 = s, a_t = \pi'(s_t),
s_{t+1} \sim \cP(\cdot | s_t, a_t)
\right],
\end{align*}
where the second identity follows from the cancelation of the terms
by taking telescoping sum, and the last identity follows from  
$\bbe\left[V^\pi(s_0)\mid
s_0 = s, a_t = \pi'(s_t),
s_{t+1} \sim \cP(\cdot | s_t, a_t)
\right] = V^\pi(s).$
Now using the above identity, \eqnok{eq:QV2} and \eqnok{eq:def_visit_measure},
we have 
\begin{align*}
V^{\pi'}(s) - V^\pi(s)
&=  \bbe
 \left[
 \tsum_{t=0}^\infty \gamma^t [Q^\pi(s_t, a_t) - V^\pi(s_t) +  h^{\pi'(s_t)}(s_t) - h^{\pi(s_t)}(s_t)] \right.\\
 &\quad \quad \quad \left. \mid
s_0 = s, a_t = \pi'(s_t),
s_{t+1} \sim \cP(\cdot | s_t, a_t)
\right]\\
&= \tfrac{1}{1-\gamma}\int [Q^\pi(q, \pi'(q)) - V^\pi(q) +  h^{\pi'(q)}(q) - h^{\pi(q)}(q) ] \, \kappa^{\pi'}_s(d q).
\end{align*}
\end{proof}

We call $\psi^{\pi}$ in \eqnok{eq:def_advantage} the \textit{advantage function}.
It is different than the advantage function used in the RL literature in that it deals with deterministic policies and
involves the term $h^{a}$.

\begin{lemma} \label{lem:monotonicity}
For any feasible policy $\pi$, we have
$
 \bbe_{s \sim \nu^*}\left[ -\psi^\pi(s, \pi^*(s)) \right]
 =  (1-\gamma) \bbe_{s\sim \nu^*}[V^\pi(s) - V^{\pi^*}(s)].
$
\end{lemma}
 
  \begin{proof}
  By~\cref{lem:performance_diff_deter} (with $\pi' = \pi^*$),
    $(1-\gamma) [V^{\pi^*}(s) - V^\pi(s)]
    = \bbe_{q\sim \kappa_s^{\pi^*}}\left[ \psi^\pi(q, \pi^*(q)) \right]$.
 Taking expectation for w.r.t. $\nu^*$ and noticing the stationarity of $\nu^*$, 
 we have
 \begin{align*}
  (1-\gamma) \bbe_{s\sim \nu^*}[V^{\pi^*}(s) - V^\pi(s)]
  &= \bbe_{s\sim \nu^*, q\sim \kappa_s^{\pi^*}} \left[ \psi^\pi(q, \pi^*(q))  \right]
  = \bbe_{s \sim \nu^*}\left[ \psi^\pi(s, \pi^*(s))\right].
 \end{align*}
 \end{proof}
 
 When both the state and action space are finite,~\cref{lem:performance_diff_deter} can be used to establish the gradient 
 of the objective function $f$ in \eqnok{eq:RL0} w.r.t.~policy $\pi$~\cite{LiYan2022Robust,LanPMD2021}.
For general state-action spaces, the policy resides in an infinite-dimensional space, making the establishment of a gradient significantly more complex than in the finite setting. In fact, it is unclear whether a gradient even exists in this context, which poses challenges for defining optimality conditions. However, ~\cref{lem:monotonicity} suggests that the advantage function may serve as an alternative measure of local optimality, even in the absence of gradients. Moreover, this criterion is applicable to both policy mirror descent and policy dual averaging, algorithms to be studied in the next two sections, whereas gradient-based conditions, if they exist, appear to be relevant only to the former algorithm.

\section{Policy Mirror Descent} \label{sec:s2}

The policy mirror descent method was initially designed 
for solving RL problems over finite state and finite action spaces.
Our goal in this section is to generalize this algorithm 
for RL over general state and action spaces.

\subsection{The Generic Algorithmic Scheme} \label{sec:generic_PMD}

For a given policy $\pi_k$, 
the PMD method minimizes the advantage function $\psi^{\pi_k}(s, a)$
plus the regularization term $D(\pi_k(s), a)$ with a certain stepsize parameter $\eta_k\ge 0$.
In order to guarantee that each step of PMD is well-defined, we need to assume through this 
section that the policy value function $Q^\pi(s, a)$ satisfies certain curvature conditions.
In particular, we assume that for some finite $\mu_Q \ge 0$, the 
function 
\beq \label{eq:convexity_Q}
Q^\pi(s, a) + \mu_Q D(\pi(s), a)
\eeq 
is convex w.r.t. $a$ for any $s \in \cS$ and any $\pi \in \Pi$.
$Q^\pi(s, \cdot)$ is convex if $\mu_Q = 0$, and strongly convex
if $\mu_Q < 0$. 
\revise{Otherwise, if $\mu_Q > 0$, then $Q^\pi(s, \cdot)$ is nonconvex with bounded lower curvature. 
For simplicity, we refer to this class of functions as weakly convex.
This class includes differentiable functions with Lipschitz continuous gradients~\cite[Lemma 4.2]{drusvyatskiy2019efficiency}.}
Later in~\cref{sec:general_action_PMD} when we incorporate function approximation, we only require the approximation of $Q^{\pi}$ (and not $Q^{\pi}$ itself) to be weakly convex.
We introduce weak convexity now to gain intuition on how it affects convergence to a global or local solution.
Let us denote
\beq \label{eq:def_convexity_PDA}
\mu_d := \mu_h - \mu_Q.
\eeq
The advantage function $\psi^{\pi}(s, \cdot)$ defined in \eqnok{eq:def_advantage}
is convex if $\mu_d \ge 0$, and weakly convex otherwise.

\begin{algorithm}[H]
\caption{The policy mirror descent method}
\begin{algorithmic}
\State {\bf Input:} $\eta_k \ge 0$ and $\pi_0$.
\For {$k =0,1,\ldots,$}
\begin{align}
\pi_{k+1}(s) &=  \argmin_{a \in \cA} \{  \psi^{\pi_k}(s, a)  + \tfrac{1}{\eta_k}D(\pi_k(s), a)\}  \label{eq:PMD_step0}  \\
&= \argmin_{a \in \cA} \{  Q^{\pi_k}(s, a) + h^a(s)  + \tfrac{1}{\eta_k}D(\pi_k(s), a)\},  \label{eq:PMD_step}
\forall s \in \cS.
\end{align}
\EndFor
\end{algorithmic} \label{alg:basic_pmd}
\end{algorithm}

With a properly chosen $\eta_k$, the objective
function of the subproblem in \eqnok{eq:PMD_step0}
is strongly convex, and thus 
the new policy $\pi_{k+1}$ is well-defined.
The following result is well-known for mirror descent methods.

\begin{lemma} \label{lem:pmd_three_point}
If $\eta_k$ in \eqnok{eq:PMD_step0} satisfies
\beq \label{eq:PMDsteprule0}
\mu_d + \tfrac{1}{\eta_k} \ge 0,
\eeq
then for any $a \in \cA$,
\beq \label{eq:PMD_three_point}
 \psi^{\pi_k}(s, \pi_{k+1}(s))  + \tfrac{1}{\eta_k}D(\pi_k(s), \pi_{k+1}(s)) + (\mu_d + \tfrac{1}{\eta_k}) D(\pi_{k+1}(s), a)
 \le  \psi^{\pi_k}(s, a) +  \tfrac{1}{\eta_k}D(\pi_k(s), a).
\eeq
\end{lemma}

\begin{proof}
Observe that the objective function of \eqnok{eq:PMD_step0}
 is strongly convex w.r.t. $a \in \cA$ with modulus $\mu_d + 1/\eta_k$.
The result then follows immediately from the optimality condition of \eqnok{eq:PMD_step0}\cite[Lemma 3.5]{LanBook2020}.
\end{proof}

\vgap

We now show the progress made by each iteration of the PMD method.
\begin{proposition}
Suppose $\eta_k$ in the PMD method satisfies \eqnok{eq:PMDsteprule0}.
Then for any $s \in S$, 
\begin{align}
V^{\pi_{k+1}}(s) - V^{\pi_{k}}(s) \le 
\psi^{\pi_{k}}(s, \pi_{k+1}(s))
\le -[\tfrac{1}{\eta_k}D(\pi_k(s), \pi_{k+1}(s)) + (\mu_d + \tfrac{1}{\eta_k}) D(\pi_{k+1}(s), \pi_k(s))]. \label{eq:function_decrease_PMD}
\end{align}
\end{proposition}

\begin{proof}
By \eqnok{eq:PMD_three_point} with $a = \pi_k(s)$, we have
\begin{align}
&\psi^{\pi_k}(s, \pi_{k+1}(s))  + \tfrac{1}{\eta_k}D(\pi_k(s), \pi_{k+1}(s)) + (\mu_d + \tfrac{1}{\eta_k}) D(\pi_{k+1}(s), \pi_k(s)) \nn\\
& \le  \psi^{\pi_k}(s, \pi_k(s)) +  \tfrac{1}{\eta_k}D(\pi_k(s), \pi_k(s)) = 0,\label{eq:PMD_bnd_Adv1}
\end{align}
where the last identity follows from the fact that $\psi^{\pi_k}(s, \pi_k(s)) = 0$ due to \eqnok{eq:QV1} and \eqnok{eq:def_advantage}.
By~\cref{lem:performance_diff_deter}, \eqnok{eq:PMD_bnd_Adv1},
and the fact that $\kappa_s^{\pi_{k+1}} (\{s\}) \ge 1-\gamma$ due to \eqnok{eq:def_visit_measure},
we have 
\begin{align}
V^{\pi_{k+1}}(s) - V^{\pi_{k}}(s)
&= \tfrac{1}{1-\gamma} \int \psi^{\pi_{k}}(q, \pi_{k+1}(q)) \kappa_s^{\pi_{k+1}} (d q)  \nn\\
&\le \tfrac{1}{1-\gamma} \psi^{\pi_{k}}(s, \pi_{k+1}(s)) \, \kappa_s^{\pi_{k+1}} (\{s\})  
\le \psi^{\pi_{k}}(s, \pi_{k+1}(s)). \label{eq:PMD_bnd_Adv2}
\end{align}
The result then follows by combining \eqnok{eq:PMD_bnd_Adv1} with \eqnok{eq:PMD_bnd_Adv2}.
\end{proof}

\vgap

In the following two results, we establish the convergence of the PMD method
under the assumption $\mu_d \ge 0$ and $\mu_d < 0$, respectively.
\begin{theorem} \label{the:PMD_linear}
Suppose that $\mu_d \ge 0$ and $\eta_k$ in the PMD method satisfies \eqnok{eq:PMDsteprule0}.
Then for any $k \ge 1$, 
\begin{align}
\tsum_{t=0}^{k-1}\left\{\eta_t[f(\pi_{t+1}) - f(\pi^*)] - \gamma \eta_t [f(\pi_{t}) - f(\pi^*)]\right\}
+ (\mu_d \eta_{k-1} + 1) {\cD}(\pi_{k}, \pi^*) 
\le {\cD}(\pi_0, \pi^*), \label{eq:PMD_linear_rate}
\end{align}
where 
\begin{align} \label{eq:cD_defn}
   {\cD}(\pi, \pi') := \bbe_{s \sim \nu^*} [D(\pi(s), \pi'(s))]. 
\end{align}
In particular, if $\eta_k = \gamma^{-k}$, then
\[
f(\pi_k) - f(\pi^*) + (\mu_d + \gamma^{k-1}) {\cD}(\pi_{k}, \pi^*) 
\le \gamma^k [f(\pi_0) - f(\pi^*)] + \gamma^{k-1} {\cD}(\pi_{0}, \pi^*).
\]
\end{theorem}
\begin{proof}
Multiplying \eqnok{eq:PMD_three_point} by $\eta_k$ and then taking the telescopic sum (in view of $\mu_d \ge 0$), we obtain
\begin{align*}
&\tsum_{t=0}^{k-1} \eta_t \psi^{\pi_t}(s, \pi_{t+1}(s))  + \tsum_{t=0}^{k-1} D(\pi_t(s), \pi_{t+1}(s)) + (\mu_d \eta_{k-1} + 1) D(\pi_{k}(s), a) \\
 &\le \tsum_{t=0}^{k-1} \eta_t \psi^{\pi_t}(s, a) +  D(\pi_0(s), a).
\end{align*}
Setting $a = \pi^*(s)$ in the above inequality, and using the first relation in \eqnok{eq:function_decrease_PMD},
we then have
\begin{align*}
\tsum_{t=0}^{k-1}\eta_t [V^{\pi_{t+1}}(s) - V^{\pi_{t}}(s)]
+ (\mu_d \eta_{k-1} + 1) D(\pi_{k}(s), \pi^*(s)) 
\le \tsum_{t=0}^{k-1} \eta_t \psi^{\pi_t}(s, \pi^*(s)) +  D(\pi_0(s), \pi^*(s)).
\end{align*}
Now taking expectation of the above inequality w.r.t. $s \sim \nu^*$, applying~\cref{lem:monotonicity},
and using the definition of $f$ in \eqnok{eq:MDPNLP}, we then obtain
\begin{align*}
\tsum_{t=0}^{k-1}\eta_t[f(\pi_{t+1}) - f(\pi_{t})]
+ (\mu_d \eta_{k-1} +1) {\cD}(\pi_{k}, \pi^*) 
\le (1-\gamma) \tsum_{t=0}^{k-1}\eta_t [f(\pi^*) - f(\pi_t)] + {\cD}(\pi_0, \pi^*),
\end{align*}
which clearly implies \eqnok{eq:PMD_linear_rate} after rearranging the terms.
\end{proof}

In view of~\cref{the:PMD_linear}, the PMD method will exhibit 
a linear rate of convergence if the advantage function is convex w.r.t. $a \in \cA$.
This result generalizes the previous linear rate of convergence for the PMD method in
the tabular setting to more general problems.
\revise{Our next result supplies convergence guarantees when $\mu_d < 0$, which will be discussed in more detail after the proof.}

\begin{theorem} \label{the:pmd_nonconvex}
Suppose that $\mu_d <0$ and that
$\eta_k = 1/ (2|\mu_d| )$ for any $k \ge 0$. 
Then
for any $s \in \cS$, there exists an iteration index $\bar{k}(s)$ found by running $k$ iterations of the PMD method s.t.~
\begin{align}
0 \leq -\psi^{\pi_{\bar{k}(s)}}(s, \pi_{\bar{k}(s)+1}(s)) 
&\leq \tfrac{1}{k}[V^{\pi_0}(s) - V^{\pi^*}(s)] 
\label{eq:stationary_PMD0} \\
2D\big(\pi_{\bar{k}(s)}(s), \pi_{\bar{k}(s)+1}(s)\big) + D\big(\pi_{\bar{k}(s)+1}(s), \pi_{\bar{k}(s)}(s)\big) 
&\leq
\tfrac{1}{k \vert \mu_d \vert}[V^{\pi_0}(s) - V^{\pi^*}(s)]. \label{eq:stationary_PMD1}
\end{align}
\end{theorem}

\begin{proof}
Taking the telescopic sum of~\eqnok{eq:function_decrease_PMD},
we obtain
\begin{align*}
\tsum_{t=0}^{k-1}[-\psi^{\pi_{t}}(s, \pi_{t+1}(s))]  
\le V^{\pi_{0}}(s) -V^{\pi_{k}}(s) \le V^{\pi_{0}}(s) -V^{\pi^*}(s), \quad \forall s \in \cS.
\end{align*} 
For a state $s \in \cS$, the iteration index $\bar{k}(s)$ in~\eqnok{eq:stationary_PMD0} is given by the one with the smallest value of  $[-\psi^{\pi_t}(s, \pi_{t+1}(s)]$ for $t =0,\ldots,k-1$. The relation~\eqnok{eq:stationary_PMD1} follows directly from \eqnok{eq:function_decrease_PMD}
and~\eqnok{eq:stationary_PMD0}.
\end{proof}

\vgap

\revise{We will now use the theorem to derive a new notion of stationarity under certain regularity conditions. Suppose $\psi^{\pi_{\bar{k}(s)}}(s,\cdot)$ is differentiable and $L$-smooth, the distance generating function for Bregman's distance is $L_\omega$-smooth, and $\cA$ is compact. Let $\epsilon = k^{-1/2}$. By combining these assumptions, $\|\pi_{\bar{k}(s)}(s) - \pi_{\bar{k}(s)+1}(s)\| \leq O(\epsilon)$ (which can be shown by~\eqnok{eq:stationary_PMD1} and~\eqnok{eq:omega_strong_convexity}), and~\cite[Lemma 6.3]{LanBook2020}, one can then derive}
\begin{align} \label{eq:stationary_pt}
    \langle \nabla_a \psi^{\pi_{\bar{k}(s)}}(s, \pi_{\bar{k}(s)}(s)), a - \pi_{\bar{k}(s)}(s) \rangle \geq -O(\epsilon), \quad \forall a \in \cA,
\end{align}
\revise{where the big-O hides dependence on $(1-\gamma)^{-1}$, $\vert \mu_d \vert$, $L$, $L_\omega$, and $\vert \cA \vert$.
We say a policy $\pi_{\bar{k}(s)}$ is at an \textit{$O(\epsilon)$-stationary point} if it satisfies~\eqnok{eq:stationary_pt}. In some cases, convergence guarantees can be derived from~\eqref{eq:stationary_pt}. When the number of actions is a finite number $\nA$, we have $\nabla_a \psi^{\pi}(s,a) = \psi^{\pi}(s) := \{\psi^{\pi}(s,e_a)\}_{a =1}^{\nA} \in \mathbb{R}^{\nA}$, where $e_a$ is an extreme point of $\Delta_\nA$. Then $\pi_{\bar{k}(s)}$ is an approximate solution to the variational inequality (VI),}
\begin{align*}
    \text{Find $\hat{\pi}$ s.t.}\quad \langle \psi^{\hat \pi}(s), a - \hat \pi(s) \rangle \geq 0, \quad \forall a \in \Delta_{\nA}.
\end{align*}
\revise{As seen by~\cref{lem:performance_diff_deter} (see also~\cite{LanPMD2021}), solving the VI across all states simultaneously implies a global solution to~\eqnok{eq:RL0}.
We leave the analysis of stationary points for more general RL problems, including those where the gradient $\nabla \psi^{\pi_t}(s,a)$ may not exist, to important future work. 
Another issue is that~\eqref{eq:stationary_pt} may occur at a different index $\bar{k}(s)$ for each state $s \in \cS$. The strategy of choosing a different iteration index for a different state can possibly be implemented
when function approximation is used in PMD as shown in the next subsection.}

\subsection{Function Approximation in PMD} \label{sec:func_approx_in_pmd}
In an ideal situation for the PMD method, we update a given policy $\pi_k$ by 
\begin{align}
\pi_{k+1}(s)
&= \argmin_{a \in \cA} \{ \underbrace{ Q^{\pi_k}(s, a) -  \tfrac{1}{\eta_k} \langle \nabla \omega(\pi_k(s)), a\rangle}_{L^{\pi_k}(s, a)} +  h^a(s) +  \tfrac{1}{\eta_k} \omega(a)\}, \forall s \in \cS \label{eq:def_ideal_PMD}
\end{align}
for some properly chosen stepsize $\eta_k \ge 0$. 
However, for the general state space $\cS$, it is impossible to 
perform these updates since $\cS$ is possibly infinite. Current practice is
to apply different function approximations for both estimating $Q^{\pi_k}(s, a)$
and computing a parameterized form of $\pi_k(s)$.
While the idea of using a parameterized form of $\pi_k$ appears to be conceptually simple, it often 
results in a few significant limitations. First, it introduces additional approximation error due to the difficulty 
of identifying the optimal policy class.
 Second, it will destroy the convexity of
the action space $\cA$ and make it difficult to ensure the policy's chosen action is feasible, i.e., lies in $\mathcal A$, without further processing.

In order to deal with these issues for RL with general state spaces,
we adopt the idea of using function approximation for the value functions.
However, different from existing methods that require function approximation
for both value function and policy parameterization, 
we will not use policy parameterization in the policy optimization step.
Instead, as shown in~\cref{alg:basic_inf_pmd}, we suggest to compute a function approximation 
for an augmented action-value function $L^{\pi_k}(s, a)$
(rather than $Q^{\pi_k}(s, a)$) used to define the objective function in \eqnok{eq:def_ideal_PMD}.

\begin{algorithm}[H]
\caption{The policy mirror descent method with function approximation}
\begin{algorithmic}
\State {\bf Input:} $\eta_k \ge 0$ and $\pi_0$.
\For {$k =0,1,\ldots,$}
\State Set
\beq \label{eq:aug_value_function}
L^{\pi_k}(s, a)  =  Q^{\pi_k}(s, a) +  \tfrac{1}{\eta_k} \langle \nabla \omega(\pi_k(s)), a\rangle.
\eeq
\State Run a policy evaluation procedure to estimate $L^{\pi_k}(s, a)$ s.t.
\beq \label{eq:policy_eval}
\tilde L(s, a; \theta_k) \approx  L^{\pi_k}(s, a), \forall s \in \cS, \forall a \in \cA.
\eeq  
\State Update the policy to
\beq \label{eq:PMD_step_fa}
\pi_{k+1}(s) = \argmin_{a \in \cA} \{ \tilde L(s, a; \theta_k) + h^a(s) +  \tfrac{1}{\eta_k} \omega(a)\},
\forall s \in \cS.
\eeq
\EndFor
\end{algorithmic} \label{alg:basic_inf_pmd}
\end{algorithm}

In \eqnok{eq:policy_eval}, $\tilde L(s, a; \theta_k)$ represents a function approximation, parameterized by $\theta_k$, of the augmented action-value function  $L^{\pi_k}(s, a)$. 
We assume that $\tilde L(s, a; \theta_k)$ is given in the form akin to~\eqnok{eq:aug_value_function},
\beq \label{eq:app_pmd_continuous}
\tilde L(s, a; \theta_k) = \tilde Q(s, a;\theta_k) + \langle \tilde \nabla \omega(s; \theta_k), a \rangle,
\eeq
where $\tilde Q(s, a;\theta_k)$ and $\tilde \nabla \omega(s; \theta_k)$ approximate $Q^{\pi_k}(s,a)$ and $\nabla \omega(\pi_k(s))/\eta_k$, respectively.
In the simplest form of function approximation of say $Q^{\pi_k}$, we are given  
a set of linearly independent basis functions $q_1(s,a), \ldots, q_d(s,a)$ 
(derived from {feature construction}~\cite{RichardBarto2018} such as random features~\cite{rahimi2007random}).
The goal then is to determine the coefficients $\theta_k \in \bbr^d$ so that 
$\tilde Q(s, a; \theta_k) = \tsum_{i=1}^d \theta_k^i q_i(s,a)$ approximates $Q^{\pi_k}(s, a)$.
More powerful and sophisticated nonlinear function approximation methods, such as Reproducing Kernel Hilbert Space (RKHS) and
neural networks, can also be used to estimate both $Q^{\pi_k}(s, a)$ and $\nabla \omega(\pi_k(s))$. 
Stochastic estimates of $Q^{\pi_k}(s,a)$ can be constructed from samples along a trajectory in the environment~\cite{li2023policy}, which are then used to fit~\eqnok{eq:app_pmd_continuous} based on the following statistical learning/estimation problem,
\begin{align}
\min_{\theta \in \Theta} \big\{\bbe_{s \sim \rho} [\|\tilde Q(s, \cdot; \theta) - Q^{\pi_k}(s,\cdot)\|^2_* + \|\tilde \nabla \omega(s; \theta) -\nabla \omega(\pi_k(s))/\eta_k\|_*^2 ]\big\},\label{eq:pmd_statistical_learning}
\end{align} 
for a given distribution $\rho$ over $\cS$, where $\|\cdot\|_*$ is the dual norm of a given arbitrary norm $\|\cdot\|$ of the action space. 
We do not specify $\rho$ here, since it can take on multiple forms, such as the uniform distribution over states (i.e., off-policy sampling) or the steady state distribution of $\pi_k$ (i.e., on-policy sampling).
We can make a couple observations.
First,~\eqnok{eq:pmd_statistical_learning} is a convex optimization problem if $\tilde Q(s, a; \theta)$ and $\tilde \nabla \omega(s; \theta)$ are linear w.r.t. $\theta$.
Second, observe that there exists two types of uncertainty for the stochastic optimization problem in \eqnok{eq:pmd_statistical_learning}, including: a) the sampling of $s$ from $\cS$ according to $\rho$ and b) the estimation of $Q^{\pi_k}(s,a)$ through the sampling of the Markov Chain induced by $\pi_k$.
We use $\xi_k$ to denote these random samples generated at iteration $k$ to solve problem \eqnok{eq:pmd_statistical_learning},
and $F_k := \{\xi_1, \ldots, \xi_k\}$.
Let $\theta_k^*$ be an optimal solution of \eqnok{eq:pmd_statistical_learning} and $\theta_k$ be an approximate solution
found by using some policy evaluation algorithm.
We will discuss such algorithms in~\cref{sec:appendix_pe}.

It is worth noting that the stepsize $\eta_k$ appears in~\eqnok{eq:pmd_statistical_learning}, which may impact the magnitude of the stochastic first-order information when solving~\eqnok{eq:pmd_statistical_learning} via a stochastic approximation method.
As a result, the sample complexity of solving this statistical learning problem might increase as $\eta_k$ decreases. 
One can possibly alleviate this issue by using a constant stepsize $\eta_k$, but this may worsen the dependence on $1-\gamma$ as noticed in~\cite{LanPMD2021}.
Another possible solution is to separately estimate $\nabla \omega(\pi_k(s))/\eta_k$ and $Q^{\pi_k}$ and use a larger sample size to approximate the former term, since estimating the former can be done faster since it does not involve samples from the Markov chain associated with $\pi_k$.
In the next section, we will propose the policy dual averaging method to avoid estimation of $\nabla \omega(\pi_k(s))\eta_k$.

The equation in \eqnok{eq:PMD_step_fa} tells us how to compute the new policy $\pi_{k+1}$.
In some cases, $\pi_{k+1}(s)$ has an explicit form.
For example, consider RL problems with a finite set of actions $\{\cA_1, \ldots, \cA_{\nA}\}$.
As derived earlier in the text surrounding~\eqnok{eq:def_simplex}, we search for an optimal randomized policy in the probability simplex $\cA = \Delta_\nA$. 
\revise{Thus, the action space becomes a vector space.}
The cost and transition probability then follow the definition from~\eqnok{eq:dis_kernel}.
For a given iterate $\pi_k(s) \in \cA$, we write
\begin{align} 
Q_0^{\pi_k}(s, \cA_i) &:= c_0(s, \cA_i) + h^{\pi_k(s)}(s) +  \gamma \int \cP_0(ds'| s,\cA_i) V^{\pi_k}(s'), i = 1,\ldots, \nA, \label{eq:finite_Q_function}
\end{align}
as the value function associated with action $\cA_i$.
We also denote 
$Q_0^{\pi_k}(s) = (Q_0^{\pi_k}(s, \cA_1), \ldots,Q_0^{\pi_k}(s, \cA_\nA))$ 
for notational convenience.
It then follows that
$Q^{\pi_k}(s, a) = \langle Q_0^{\pi_k}(s), a \rangle$ and therefore $L^{\pi_k}(s,a) = \langle Q_0^{\pi_k}(s) + \frac{1}{\eta_k}\nabla \omega(\pi_k(s)), a \rangle$.
We suggest to approximate $Q_0^{\pi_k}(s, \cA_i)$ by $\tilde Q_0(s, \cA_i; \theta_k)$, $i = 1, \ldots, \nA$, 
where $\tilde Q_0(s, \cA_i; \theta)$ denotes some function approximation parameterized by $\theta$. 
Then the subproblem in~\eqnok{eq:PMD_step_fa} reduces to,
\beq \label{eq:PMD_step_finite}
\pi_{k+1}(s) = \argmin_{a \in \Delta_\nA} \{ \underbrace{\langle \tilde Q_0(s; \theta_k) + \tilde \nabla \omega(s;\theta), a \rangle}_{\tilde L(s,a;\theta_k)} + h^a(s) + \tfrac{1}{\eta_k}\omega(a)\},
\forall s \in \cS,
\eeq
where
\beq \label{eq:approx_Q_vec}
    \tilde Q_0(s;\theta_k) := \big(\tilde{Q}_0(s,\cA_1;\theta_k), \ldots, \tilde{Q}_0(s,\cA_{\nA};\theta_k)\big).
\eeq
The objective function of \eqnok{eq:PMD_step_finite} is strongly convex with modulus $\tilde \mu_d = \mu_h + 1/\eta_k$.
With choice of $\omega(a) =\tsum_{i=1}^\nA a_i \log a_i$ and $h^a = \lambda {\rm KL}(a||\pi_0(s))$ for some $\lambda \ge 0$, it is well-known that the above problem has an explicit solution~\cite{LanPMD2021}.

In a more general setting, such as when $\cA$ is not given as a finite set, $\pi_{k+1}(s)$ is represented as an optimal solution to the problem~\eqnok{eq:PMD_step_fa}.
With a properly chosen $\eta_k$, the objective function of \eqnok{eq:PMD_step_fa}
is strongly convex, and hence the policy $\pi_{k+1}(s)$ is uniquely defined.
Moreover, any efficient numerical procedures (e.g., the Ellipsoid method and accelerated gradient method) can be used to solve
\eqnok{eq:PMD_step_fa}. 
For the sake of simplicity, we assume throughout this paper
that exact solutions can be computed for \eqnok{eq:PMD_step_fa}. However, our analysis can be extended 
to the case when only approximate solutions are available.
See Section~6 of \cite{LanPMD2021} for an illustration.
Observe that,
different from existing policy gradient methods,
we do not actually compute $\pi_{k+1}(s)$ for any $s \in \cS$.
Since the function approximation $L(s, a; \theta_k)$ is given in a parametric form (through $\theta_k$'s),
we can compute $\pi_{k+1}(s)$ using \eqnok{eq:PMD_step_fa} whenever needed, e.g., inside
the policy evaluation procedure. The idea of representing
a policy by an optimization problem has been used in 
multi-stage stochastic optimization.

\subsection{PMD for General State and Continuous/Finite Action Spaces} \label{sec:general_action_PMD}
We now investigate the impact of the approximation error from~\eqnok{eq:def_app_error_finite} on the convergence of the PMD method for solving RL problems.
We still assume that the approximate augmented action-value function $\tilde L(s,a;\theta_k)$ is given in the form~\eqnok{eq:app_pmd_continuous}. 
We further assume that $\tilde Q(s, a;\theta)$ is weakly convex, i.e., 
\beq \label{eq:weakly_convex}
\tilde Q(s, a; \theta) + \mu_{\tilde Q} D(\pi_0(s), a) 
\eeq
is convex w.r.t. $a \in \cA$ for some $\mu_{\tilde Q} \ge 0$, and we denote
$
\tilde \mu_d := \mu_h - \mu_{\tilde Q}.
$
Moreover, we assume that $\tilde Q(s, \cdot; \theta)$ and $Q^\pi(s, \cdot)$ are Lipschitz continuous
s.t.
\begin{align}
|\tilde Q(s, a_1; \theta) - \tilde Q(s, a_2; \theta)| \le M_{\tilde Q} \|a_1 - a_2\|, \forall a_1, a_2 \in \cA, \label{eq:Lip_psi} \\
|Q^\pi(s, a_1) - Q^\pi(s, a_1)| \le M_Q \|a_1 - a_2\|, \forall a_1, a_2 \in \cA. \label{eq:Lip_Q}
\end{align}
We make a couple remarks.
First, we require~\eqnok{eq:Lip_psi} to hold w.r.t.~any $\theta \in \Theta$, where $\Theta$ is some feasible class of parameters.
Second, note that the Lipschitz constant of  $L^{\pi_k}(s,a)$ can be very large,  since $\eta_k$ can possibly be small.
That explains why we need to handle the two terms $ \tilde Q(s, a;\theta_k)$ and $\langle \tilde \nabla \omega(s; \theta_k), a \rangle$  in the definition of $\tilde L$ separately.
Third, the aforementioned assumptions hold if $Q^\pi(s,\cdot)$ is continuous and its approximation $\tilde{Q}(s,\cdot;\theta_k)$ is continuous and Lipschitz smooth for any $s \in \cS$~\cite[Lemma 4.2]{drusvyatskiy2019efficiency}.

In this section, we re-write the approximation error as
\beq \label{eq:def_app_error_finite}
\delta_k(s,a) = \tilde L(s,a; \theta_k) - L^{\pi_k}(s,a) = \delta_k^Q(s,a) + \langle \delta_k^\omega(s), a \rangle,
\eeq
where the individual summands are defined as
\begin{align}
\delta_k^Q(s, a) &:= \tilde Q(s, a; \theta_k) - Q^{\pi_k}(s,a) = \underbrace{\mathbb{E}[\tilde Q(s,a; \theta_k)] - Q^{\pi_k}(s,a)}_{\delta^{Q,det}_k(s,a)} + \underbrace{\tilde Q(s, a; \theta_k) - \mathbb{E}[\tilde Q(s,a; \theta_k)]}_{\delta^{Q,sto}_k(s,a)}, \label{eq:delta_Q_decomp} \\
\delta_k^\omega(s) &:= \tilde \nabla \omega(s; \theta_k) -\tfrac{1}{\eta_k} \nabla \omega(\pi_k(s))
= \underbrace{\bbe[\tilde \nabla \omega(s; \theta_k)] - \tfrac{1}{\eta_k} \nabla \omega(\pi_k(s))}_{\delta_k^{\omega,det}(s) }
+  \underbrace{ \tilde \nabla \omega(s; \theta_k)  - \bbe[\tilde \nabla \omega(s; \theta_k)]}_{\delta_k^{\omega,sto}(s) }, \nonumber
\end{align}
and the expectation $\bbe_{\xi_k}$ is taken w.r.t.~$\xi_k$ conditioned on $F_{k-1}$.
Note that $\delta_k^{Q,det}(s)$ and $\delta_k^{Q,sto}(s)$ denote the deterministic and stochastic errors, respectively, and likewise with $\delta_k^{\omega,det}(s)$ and $\delta_k^{Q,sto}(s)$.
Moreover, the deterministic error $\delta_k^{Q,det}(s)$ can be further decomposed into
\begin{align} 
\delta_k^{Q,det}(s,a) 
&= \underbrace{\tilde Q_0(s,a; \theta_k^*) - Q_0^{\pi_k}(s,a)}_{\delta_k^{Q,app}(s,a)}
+\underbrace{\bbe_{\xi_k}[\tilde Q(s,a; \theta_k)] - \tilde Q(s,a; \theta_k^*)}_{\delta_k^{Q,bia}(s,a)}, \label{eq:def_decomp_error_finite}
\end{align}
where $\delta_k^{Q,app}$ is the error caused by function approximation and $\delta_k^{Q,bia}(s)$ is the bias term of the estimator. 
A nearly identical decomposition exists for the error $\delta_k^{\omega,det}(s)$.
One can possibly reduce $\delta_k^{Q,app}(s)$ by using more complex function classes.
In particular, we have $\delta_k^{Q,app}(s) = 0$ in the tabular case and for some specific kernels from RKHS. 
While these two deterministic error terms are caused by different sources, their impact on the overall convergence
 of the PMD method is similar.   

To facilitate our analysis, we need to make assumptions on the aforementioned error terms.
We start with the convex case when $\tilde \mu_d \ge 0$.
We assume the approximation error $\delta_k^Q(s,a)$ satisfies
$\delta_k^\omega(s)$:
\begin{align}
\bbe_{s \sim \nu^*}[|\delta^{Q,det}_{k}(s, \pi_{k}(s))| + |\delta^{Q,det}_{k}(s, \pi^*(s))|] \le \varsigma^Q,\label{eq:weak_ass_pmd_convex1} \\
\bbe_{s \sim \nu^*}[ \|\delta_k^{\omega,det}(s)\|_*]
\le \varsigma^\omega, \label{eq:weak_ass_pmd_convex2}\\
\bbe_{\xi_k}\left\{\bbe_{s \sim \nu^*}[ \|\delta_k^{\omega,sto}(s)\|_*^2]\right\}  \le  (\sigma^{\omega})^2, \label{eq:weak_ass_pmd_convex3}
\end{align}
for some $\varsigma^Q \geq 0$, $\varsigma^\omega \geq 0$, and $\sigma^\omega \geq 0$.
The terms $\varsigma^Q$ and $\varsigma^\omega$ capture the approximation errors as mentioned earlier, while the term $\sigma^\omega$ bounds the statistical estimation error of $\delta_k^{\omega,sto}(s)$.
Both the $\delta^{Q,bia}(s)$ and $\delta^{Q,sto}(s)$ terms become smaller with more samples, although $\delta^{Q,bia}(s)$ can usually be reduced much faster than $\delta^{sto}(s)$ as observed previously in \cite{LanPMD2021}.
\edits{Additionally, notice that we do not need a bound on the stochastic error term $\delta^{Q,sto}_k$.
Later when we specialize our main convergence results to the finite action case, we will use a bound on $\delta^{Q,sto}_k$ solely for the purpose of simplifying error terms.}
In~\cref{sec:appendix_pe}, we give sufficient conditions for these assumptions to hold.

With these assumptions set in place, we begin our convergence analysis. 
The following result states an important recursion about each PMD iteration.
\begin{proposition}
For any $s \in \cS$, we have
\begin{align}
&- \psi^{\pi_k}(s,a) +  (\tilde \mu_d+ \tfrac{1}{\eta_k})D(\pi_{k+1}(s), a) \nn\\
&\le \tfrac{1}{\eta_k}  D(\pi_k(s), a) + \eta_k [ 2M_Q + M_{\tilde Q} + M_h + \| \delta_k^{\omega,sto}(s)\|_*]^2  \nn
\\ &\hspace{10pt}
+\|\delta_k^{\omega,det}(s)\|_*\|a-\pi_{k+1}(s)\|
- \langle \delta_k^{\omega,sto}(s), \pi_{k}(s) - a \rangle - \delta_k^Q(s,\pi_k(s)) + \delta_k^Q(s,a), \quad \forall a \in \cA. \label{eq:pmd_finite_recursion}
\end{align}
\end{proposition}

\begin{proof}
By the optimality condition of \eqnok{eq:PMD_step_fa}, we have
\[
\tilde L(s, \pi_{k+1}(s);\theta_k) - \tilde L(s, a; \theta_k) + h^{\pi_{k+1}(s)}(s) - h^a(s) + 
\tfrac{1}{\eta_k}[\omega(\pi_{k+1}(s)) - \omega(a)] + (\tilde \mu_d + \tfrac{1}{\eta_k})D(\pi_{k+1}(s), a) \le 0,
\]
which implies that
\begin{align}
&L_k(s, \pi_{k+1}(s)) - L_k(s,a) +  [h^{\pi_{k+1}(s)}(s) - h^a(s)] + \tfrac{1}{\eta_k} [\omega(\pi_{k+1}(s)) - \omega(a)] +  (\tilde \mu_d + \tfrac{1}{\eta_k})D(\pi_{k+1}(s), a) \nn\\
&\quad +\delta_k^Q(s, \pi_{k+1}(s)) - \delta_k^Q(s, a) + \langle \delta_k^\omega(s), \pi_{k+1}(s) - a \rangle \le 0.
\end{align}
Using~\eqnok{eq:aug_value_function} and noticing that 
\[
\omega(\pi_{k+1}(s)) - \omega(a) + \langle \nabla \omega(\pi_k(s)), \pi_{k+1}(s) - a\rangle 
= D(\pi_{k}(s), \pi_{k+1}(s)) - D(\pi_k(s), a),
\]
we have
\begin{align}
&Q^{\pi_k}(s, \pi_{k+1}(s)) - Q^{\pi_k}(s,a)+ h^{\pi_{k+1}(s)}(s) - h^a(s) + \tfrac{1}{\eta_k} [D(\pi_{k}(s), \pi_{k+1}(s)) - D(\pi_k(s), a)] \nn\\
& \quad +  (\tilde \mu_d+ \tfrac{1}{\eta_k})D(\pi_{k+1}(s), a) +\delta_k^Q(s, \pi_{k+1}(s)) - \delta_k^Q(s, a) + \langle \delta_k^\omega(s), \pi_{k+1}(s) - a \rangle \le 0, \nn
\end{align}
which, in view of \eqnok{eq:def_advantage}, then implies that
\begin{align}
&\psi^{\pi_k}(s, \pi_{k+1}(s)) - \psi^{\pi_k}(s,a) + \tfrac{1}{\eta_k} [D(\pi_{k}(s), \pi_{k+1}(s)) - D(\pi_k(s), a) ]\nn\\
& \quad +  (\tilde \mu_d + \tfrac{1} {\eta_k})D(\pi_{k+1}(s), a) +\delta_k^Q(s, \pi_{k+1}(s)) - \delta_k^Q(s, a) + \langle \delta_k^\omega(s), \pi_{k+1}(s) - a \rangle \le 0. \label{eq:pmd_continuous_state_recursion}
\end{align}
Noticing that
\begin{align}
\psi^{\pi_k}(s, \pi_{k+1}(s)) 
&= [Q^{\pi_k}(s, \pi_{k+1}(s)) - Q^{\pi_k}(s, \pi_{k}(s))] + [h^{\pi_{k+1}(s)}(s) - h^{\pi_k(s)}(s)] \nn \\
 & \ge - (M_Q + M_h)\|\pi_{k+1}(s)-\pi_{k}(s)\| \label{eq:PMD_continuous_bnd_error2}\\
 \delta_k^Q(s, \pi_{k+1}(s)) & =  \delta_k^Q(s, \pi_{k+1}(s)) - \delta_k^Q(s, \pi_{k}(s)) + \delta_k^Q(s, \pi_{k}(s)) \nn\\
 &\ge -(M_Q + M_{\tilde Q})\|\pi_{k+1}(s)-\pi_{k}(s)\| + \delta_k^Q(s, \pi_{k}(s)),  \nn 
 \end{align}
 and
 \begin{align}
 &\langle \delta_k^\omega(s), \pi_{k+1}(s) - a \rangle \nn \\
 &=  
 \langle \delta_k^{\omega,det}(s), \pi_{k+1}(s) - a \rangle + \langle \delta_k^{\omega,sto}(s), \pi_k(s) - a\rangle + \langle \delta_k^{\omega,sto}(s), \pi_{k+1}(s) - \pi_k(s)\rangle \nn\\
 &\ge - \|\delta_k^{\omega,det}(s)\|_* \edits{\|\pi_{k+1}(s) - a\|}  - \|\delta_k^{\omega,sto}(s)\|_* \|\pi_{k+1}(s) - \pi_k(s)\| + \langle \delta_k^{\omega,sto}(s), \pi_k(s)- a\rangle,  \label{eq:PMD_continuous_bnd_error4}
\end{align}
we then get the claimed result after combining with~\eqnok{eq:pmd_continuous_state_recursion}, the error decomposition in~\eqnok{eq:delta_Q_decomp}, and
\begin{align*}
    &-(2M_Q + M_{\tilde Q} + M_h + \|\delta_{k}^{\omega, sto}(s)\|_*)\|\pi_{k+1}(s) - \pi_k(s)\| + \tfrac{1}{\eta_k}D(\pi_{k}(s), \pi_{k+1}(s)) \\
    &\geq
    \tfrac{\eta_k}{2}(2M_Q + M_{\tilde Q} + M_h + \|\delta_{k}^{\omega, sto}(s)\|_*)^2,
\end{align*}
where the above inequality is due to 1-strong convexity of Bregman's distance and Young's inequality.
\end{proof}

\vgap

Observe that in the proof of~\cref{eq:pmd_finite_recursion}, we bound the term $\psi^{\pi_k}(s, \pi_{k+1}(s))$ and $h^{\pi_{k+1}(s)}(s) - h^{\pi_{k}(s)}(s)$ in~\eqnok{eq:PMD_continuous_bnd_error2} simply by the Cauchy-Schwarz inequality and Lipschitz continuity of $h^a$, respectively.
These relations are sufficient to prove the sublinear rate of convergence for the PMD method.
In order to prove the linear rate of convergence of the PMD method in terms of number of iterations
for the stochastic setting, more sophisticated techniques to bound these terms are needed as initially shown in \cite{LanPMD2021}.

\begin{theorem} \label{the:main_pmd_finite2}
Suppose \eqnok{eq:weak_ass_pmd_convex1}-\eqnok{eq:weak_ass_pmd_convex3} hold.
Also assume that the stepsize $\eta_k$ satisfies
\beq \label{eq:pmd_continuous_step}
\tfrac{\beta_k}{\eta_k} \le \beta_{k-1} (\tilde \mu_d+ \tfrac{1}{\eta_{k-1}}), k \ge 1,
\eeq
for some $\beta_k \ge 0$. Then
\begin{align*}
& [\tsum_{t=0}^{k-1} \beta_t]^{-1} \left\{(1-\gamma) \tsum_{t=0}^{k-1} \beta_t\bbe [f(\pi_t) - f^*] + \beta_{k-1} (\tilde \mu_d+ \tfrac{1}{\eta_{k-1}}) \bbe [{\cD}(\pi_{k}, \pi^*)]\right\}\\
&\le [\tsum_{t=0}^{k-1} \beta_t]^{-1} \left\{\tfrac{\beta_0}{\eta_0} {\cD}(\pi_0, \pi^*) 
+ \tsum_{t=0}^{k-1} \beta_t\eta_t [(2M_Q+ M_h)^2 + 2(M_{\tilde Q}^2 + (\sigma^{\omega})^2) \right\}
+\varsigma^Q + \varsigma^\omega \bar D_{\cA, \|\cdot\|},
\end{align*}
where ${\cD}$ is defined in~\eqnok{eq:cD_defn} and
\beq \label{eq:def_bar_DA}
\edits{\bar D_{\cA,\|\cdot\|} := \max_{a_1, a_2 \in \cA} \|a_1 - a_2\|}.
\eeq
\end{theorem}

\begin{proof}
Fixing $a = \pi^*(s)$, taking expectation w.r.t. $s \sim \nu^*$ for \eqnok{eq:pmd_finite_recursion}, and applying $(p+q)^2 \leq 2(p^2 + q^2)$, we have
\begin{align}
&\bbe_{s \sim \nu^*}\big[- \psi^{\pi_k}(s,\pi^*(s)) + (\tilde \mu_d+ \tfrac{1}{\eta_k})D(\pi_{k+1}(s), \pi^*(s)) \big]\nn\\
&\le 
\tfrac{1}{\eta_k} \bbe_{s \sim \nu^*} D(\pi_k(s), \pi^*(s)) 
+ \eta_k [ 2M_Q + M_h]^2 + 2[M_{\tilde Q} + \eta_k\bbe_{s\sim \nu^*}\| \delta_k^{\omega,sto}(s)\|_*^2] \nn
\\ &\hspace{10pt} 
+\bbe_{s \sim \nu^*}\big[\|\delta_k^{\omega,det}(s)\|_*\|\pi_{k+1}(s) - \pi^*(s)\| - \langle \delta_k^{\omega,sto}(s), \pi_{k}(s) - \pi^*(s) \rangle - \delta_k^Q(s,\pi_k(s)) + \delta_k^Q(s,\pi^*(s))\big]. \nn
\end{align}
Taking conditional expectation w.r.t.~$\xi_k$, it then follows from~\cref{lem:monotonicity}, \eqnok{eq:weak_ass_pmd_convex1}-\eqnok{eq:weak_ass_pmd_convex3}, and $\bbe_{\xi_k}[\delta_k^{Q,sto}(s)] = 0$ and $\bbe_{\xi_k}[\delta_k^{\omega,sto}(s)] = \mathbf{0}$ that
\begin{align*}
&(1-\gamma) \bbe_{s \sim \nu^*}\left\{[V^{\pi_k}(s) - V^{\pi^*}(s)] +  (\mu_h+ \tfrac{1}{\eta_k})\bbe_{\xi_k}[\bbe_{s \sim \nu^*}D(\pi_{k+1}(s), \pi^*(s))]\right\} \nn\\
&\le \tfrac{1}{\eta_k} \bbe_{s \sim \nu^*}\left\{ D(\pi_k(s), \pi^*(s)) + \eta_k [ (2M_Q + M_h)^2 + 2M_{\tilde Q} + 2(\sigma^\omega)^2]  
+\varsigma^Q + \varsigma^\omega \bar D_{\cA, \|\cdot\|} \right\}.
\end{align*}
Applying expectations w.r.t.~$F_k$ and then taking a $\beta_k$-weighted sum of the above inequalities,
it then follows from \eqnok{eq:pmd_continuous_step} and the definition of $f$ in \eqnok{eq:MDPNLP} that
\begin{align*}
&(1-\gamma) \tsum_{t=0}^{k-1} \beta_t\bbe [f(\pi_t) - f^*] + \beta_{k-1} (\tilde \mu_d+ \tfrac{1}{\eta_{k-1}}) \left\{ \bbe [\bbe_{s \sim \nu^*}D(\pi_{k}(s), \pi^*(s))]\right\}\\
&\le \tfrac{\beta_0}{\eta_0} \bbe_{s \sim \nu^*}D(\pi_0(s), \pi^*(s)) 
+ \tsum_{t=0}^{k-1} \beta_t\left[\eta_t [( 2M_Q + M_h)^2 + 2M_{\tilde Q} + 2(\sigma^\omega)^2] +\varsigma^Q + \varsigma^\omega \bar{D}_{\cA, \|\cdot\|} \right],
\end{align*}
which implies the result we are after.
\end{proof}

The following result follows immediately from~\cref{the:main_pmd_finite2} by straightforward calculations, so we omit the proof.

\begin{corollary} \label{cor:PMD_general_cvg}
Suppose~\eqnok{eq:weak_ass_pmd_convex1}-\eqnok{eq:weak_ass_pmd_convex3} are satisfied. Then we have:
\begin{itemize}
\item [a)] If $\tilde \mu_d = 0$, the number of iterations $k$ is given a priori, and 
\[
\eta_t = \sqrt{\tfrac{{\cD}(\pi_0, \pi^*)}{k[(2 M_Q + M_h)^2 + 2M_{\tilde Q}^2 +2(\sigma^\omega)^2]}} \ \ \mbox{and} \ \ \beta_t =1, t=0, \ldots, k-1,
\]
then
\[
(1-\gamma) \tfrac{1}{k} \tsum_{t=0}^{k-1} [f(\pi_t) - f^*]
\le 2 \sqrt{\tfrac{{\cD}(\pi_0, \pi^*) [(2 M_Q + M_h)^2 + 2M_{\tilde Q}^2 + 2(\sigma^\omega)^2] }{k}} + \varsigma^Q + \varsigma^\omega \bar{D}_{\cA, \|\cdot\|}.
\]

\item [b)] If $\tilde \mu_d >  0$ and
$
\eta_t = \tfrac{1}{\tilde \mu_d (t+1)} \ \ \mbox{and} \ \ \beta_t = 1, t = 0, \ldots, k-1,
$
then
\begin{align*}
&(1-\gamma) \tfrac{1}{k} \tsum_{t=0}^{k-1} [f(\pi_t) - f^*] + \tilde \mu_d   \bbe [{\cD}(\pi_{k}, \pi^*)]\\
&\le \tfrac{1}{k} \left\{\tilde \mu_d {\cD}(\pi_1, \pi^*) + \tfrac{2}{\tilde \mu_d}[(2 M_Q + M_h)^2 + 2M_{\tilde Q}^2 +2(\sigma^\omega)^2]  \log k \right\} +\varsigma^Q + \varsigma^\omega \bar{D}_{\cA, \|\cdot\|}.
\end{align*}

\item [c)] If $\tilde \mu_d > 0$ and
$
\eta_k = \tfrac{2}{\tilde \mu_d (t+1)} \ \ \mbox{and} \ \ \beta_t = t+2, t = 0, \ldots,k-1,
$
then
\begin{align*}
& \tfrac{2(1-\gamma)}{k (k+3)}  \tsum_{t=0}^{k-1} (t+2) \bbe[f(\pi_t) - f^*] + \tfrac{\tilde \mu_d}{3}  \bbe [{\cD}(\pi_{k}, \pi^*)]\\
&\le \tfrac{2}{k(k+3)} \tilde \mu_d {\cD}(\pi_0, \pi^*) 
+ \tfrac{8}{\tilde \mu_d(k+3)} [(2 M_Q + M_h)^2 + 2 M_{\tilde Q}^2 +2(\sigma^\omega)^2] 
+\varsigma^Q + \varsigma^\omega \bar{D}_{\cA, \|\cdot\|}.
\end{align*}
\end{itemize}
\end{corollary}

\vgap

Based on the corollary, we can randomly select the output solution $\pi_R$ from $\{\pi_0, \ldots, \pi_{k-1}\}$
with $\prob\{R = t\} = \beta_t / \tsum_{t=0}^{k-1} \beta_t$, $t = 0, \ldots, k-1$.
Observe that both part a) and part b) also provide a so-called regret bound, i.e., $\tsum_{k=0}^{k-1} [f(\pi_t) - f^*]$,
for RL problems.

Generally, one cannot have convenient knowledge of $\tilde \mu_d$ since it depends on the (unknown) lower curvature of the approximate action-value function $\tilde{Q}(s,\cdot;\theta)$. 
It may therefore be difficult to guarantee global convergence in practice.
In the special case when $\cA$ is given as a finite set, we can guarantee global convergence.
To setup notation, recall the action-value vectors $Q^{\pi_k}_0(s), \tilde{Q}_0(s;\theta_k) \in \mathbb{R}^{\nA}$ as described in~\eqnok{eq:approx_Q_vec} and the surrounding text.
We write $\delta_{0,k}^{Q,det}(s) := \mathbb{E}_{\xi_k}[\tilde Q_0(s;\theta_k)] - Q^{\pi_k}_0(s)$ and $\delta_{0,k}^{Q,sto}(s) := \tilde{Q}_0(s;\theta_k) - \mathbb{E}_{\xi_k}[\tilde{Q}_0(s;\theta_k)]$, which are similar to the terms appearing in~\eqnok{eq:def_decomp_error_finite}.
In the following result, we will assume
\begin{align} 
    \mathbb{E}_{s \sim \nu^*} \|\delta_{0,k}^{Q,det}(s)\|_\infty \leq \bar{\varsigma}^Q& \label{eq:pmd_variance_bound_1}\\
    \mathbb{E}_{\xi_k}\big\{ \mathbb{E}_{s \sim \nu^*} \|\delta_{0,k}^{Q,sto}(s)\|_\infty^2\big\} \leq (\bar{\sigma}^Q)^2&, \label{eq:pmd_variance_bound_2}
\end{align}
for some $\bar{\varsigma}^Q > 0$ and $\bar{\sigma}^Q > 0$.
\begin{theorem} \label{thm:pmd_finite}
    Suppose $\cA$ is given as a finite set of actions, and we extend it to continuous actions via randomized policies as detailed in~\eqnok{eq:finite_Q_function} and the surrounding discussion.
    If~\eqnok{eq:weak_ass_pmd_convex2}-\eqnok{eq:weak_ass_pmd_convex3} and~\eqnok{eq:pmd_variance_bound_1}-\eqnok{eq:pmd_variance_bound_2} are satisfied, then the convergence results from~\cref{cor:PMD_general_cvg} hold with
    \begin{align*}
        \tilde \mu_d = \mu_h, \quad M_Q = \tfrac{\bar{c}}{1-\gamma}, \quad M_{\tilde{Q}} = \bar{\sigma}^Q, \quad \bar{D}_{\cA, \|\cdot\|} = 2, \quad\text{and}~ \varsigma^Q = \bar{\varsigma}^Q \bar{D}_{\cA, \|\cdot\|},
    \end{align*}
    where $\|\cdot\| = \|\cdot\|_1$.
\end{theorem}
\begin{proof}
    Based on the PMD update in the finite case of~\eqnok{eq:PMD_step_finite}, the approximation error in~\eqnok{eq:def_app_error_finite} can be re-written as $\delta_k(s,a) = \langle \delta_{0,k}^Q(s),a \rangle + \langle \delta_k^\omega(s), a \rangle$.
    With this linear form,~\eqnok{eq:pmd_finite_recursion} can be refined into
    \begin{align*}
        &- \psi^{\pi_k}(s,a) +  (\tilde \mu_d+ \tfrac{1}{\eta_k})D(\pi_{k+1}(s), a) \nn\\
        &\le \tfrac{1}{\eta_k}  D(\pi_k(s), a) + \eta_k [ M_Q + M_h + \| \delta_{0,k}^{Q,sto}(s)\|_* + \| \delta_{k}^{\omega,sto}(s)\|_*]^2
        \\ &\hspace{10pt}
        +[\|\delta_{0,k}^{Q,det}(s)\|_* + \|\delta_k^{\omega,det}(s)\|_*]\|a-\pi_{k+1}(s)\|
        - \langle \delta_{0,k}^{Q,sto}(s) + \delta_{k}^{\omega,sto}(s), \pi_{k}(s) - a \rangle, \quad \forall a \in \cA.
    \end{align*}
    In view of~\eqnok{eq:pmd_variance_bound_2}, the rest of the proof follows similarly to that of~\cref{the:main_pmd_finite2}.

    It remains to show our choice of $\tilde{\mu}_d$, $M_Q$ and $\bar D_{\cA, \|\cdot\|}$ satisfy~\eqnok{eq:weakly_convex},~\eqref{eq:Lip_Q}, and~\eqref{eq:def_bar_DA}, respectively. 
    Since $\tilde{Q}(s,a;\theta_k) = \langle \tilde{Q}_0(s;\theta_k),a \rangle$ based on the PMD subproblem~\eqnok{eq:PMD_step_finite}, then $\tilde{Q}(s,\cdot;\theta_k)$ is linear, which implies~\eqnok{eq:weakly_convex} is satisfied with $\tilde{\mu}_d = \mu_h - \mu_{\tilde Q} = \mu_h$.
    Next, we similarly recall that $Q^{\pi_k}(s,a) = \langle Q^{\pi_k}_0(s),a \rangle$, where the $a$-th element of $Q^{\pi_k}(s)$ is $Q^{\pi_k}(s,\cA_i) \in [0,\frac{\bar{c}}{1-\gamma}]$. Then $M_Q = \frac{\bar{c}}{1-\gamma}$ satisfies~\eqnok{eq:Lip_Q}.
    Finally, since we set $\cA = \Delta_{\nA}$, then $\bar D_{\cA,\|\cdot\|} = 2$. 
\end{proof}

\vgap
This result says PMD solves general state finite action RL problems up to approximation error of size $O(\bar{\varsigma}^Q + \varsigma^\omega)$.
Moreover, the rate of convergence can be accelerated when $\mu_h > 0$, which can be checked in practice.

We now consider the case when $\tilde \mu_d <0$. 
To prove convergence results, we need to assume the estimation error terms $\delta_k^{\omega,det}(s)$ and $\delta_k^{\omega,sto}(s)$ are bounded, i.e.,
\begin{align}
\|\delta_k^{\omega,det}(s)\|_* \le \bar \varsigma^\omega \ \ \mbox{and} \ \
 \|\delta_k^{\omega,sto}(s) \|_* \le   \bar \sigma^{\omega}, \label{eq:strong_ass_pmd_nonconvex}
\end{align}
for some $\bar \varsigma^\omega \geq 0$ and $\bar \sigma^{\omega} \geq 0$.
On the other hand, we do not need to make any assumptions about the error term $\delta_k^Q(s,a)$ as long as 
it is Lipschitz continuous w.r.t. $a$.

\begin{theorem} \label{the:PMD_fun_stationary}
Let the number of iterations $k$ be fixed a priori.
Suppose that \eqnok{eq:strong_ass_pmd_nonconvex} holds.
Also 
assume that $\tilde \mu_d <0$ and that
$
\eta_t =\eta = \min\{ |\tilde \mu_d| / 2, 1/\sqrt{k}\} , t = 0, \ldots, k-1.
$
Then
for any $s \in \cS$, 
there exist iteration indices $\bar k(s)$ found by running $k$ iterations of the PMD method s.t.
\begin{align}
&-\psi^{\pi_{\bar k(s)}}(s, \pi_{\bar k(s)+1}(s)) \le \tfrac{1}{k} [V^{\pi_{0}}(s) -V^{\pi^*}(s)] 
+ \tfrac{\gamma}{1-\gamma}\big[\tfrac{ (M_h + M_Q + M_{\tilde Q} +\bar \sigma^{\omega})^2}{\sqrt{k}}  +  \bar \varsigma^\omega \bar{D}_{\cA, \|\cdot\|} \big],  \label{eq:stationary_SPMD0}\\
& \tfrac{1}{2\eta} D( \pi_{\bar k(s)}, \pi_{\bar k(s) +1}(s)) +  (\tilde \mu_d + \tfrac{1} {\eta})D( \pi_{\bar k(s)+1}(s),  \pi_{\bar k(s)}(s)) \nn\\
& \le\tfrac{1}{k} [V^{\pi_{0}}(s) -V^{\pi^*}(s)] 
+ \tfrac{1}{1-\gamma}\big[\tfrac{ (M_h + M_Q + M_{\tilde Q} +\bar \sigma^{\omega})^2}{\sqrt{k}}  +  \bar \varsigma^\omega \bar{D}_{\cA, \|\cdot\|}\big], \label{eq:stationary_SPMD}
\end{align}
where $\bar{D}_{\cA, \|\cdot\|}$ is defined in~\eqnok{eq:def_bar_DA}.
\end{theorem}

\begin{proof}
Setting $a = \pi_k(s)$ in \eqnok{eq:pmd_continuous_state_recursion}, and using the facts that $\psi^{\pi_k}(s,\pi_k(s)) = 0$ and $D(\pi_k(s), \pi_k(s))=0$, we obtain
\begin{align}
&\psi^{\pi_k}(s, \pi_{k+1}(s))  + h^{\pi_{k+1}(s)}(s) - h^{\pi_k}(s) + \tfrac{1}{\eta_k} D(\pi_{k}(s), \pi_{k+1}(s)) \nn\\
& \quad +  (\tilde \mu_d + \tfrac{1} {\eta_k})D(\pi_{k+1}(s), \pi_k(s))  +\delta_k^Q(s, \pi_{k+1}(s)) - \delta_k^Q(s, a) + \langle \delta_k^\omega(s), \pi_{k+1}(s) - a \rangle \le 0 \le 0. \nn
\end{align}
Using the above inequality, \eqnok{eq:PMD_continuous_bnd_error2}-\eqnok{eq:PMD_continuous_bnd_error4},
and Young's inequality, we then have
\begin{align}
&\psi^{\pi_k}(s, \pi_{k+1}(s))  + \tfrac{1}{2\eta_k} D(\pi_{k}(s), \pi_{k+1}(s)) +  (\tilde \mu_d + \tfrac{1} {\eta_k})D(\pi_{k+1}(s), \pi_k(s))  \nn\\
&\le (M_h + M_Q + M_{\tilde Q} + \|\delta_k^{\omega,sto}(s)\|_*) \|_* \|\pi_{k+1}(s) - \pi_k(s)\| - \tfrac{1}{2\eta_k} D(\pi_{k}(s), \pi_{k+1}(s))
+ \|\delta_k^{\omega,det}(s)\|_* \bar{D}_{\cA, \|\cdot\|} \nn\\
 &\le \eta_k (M_h + M_Q + M_{\tilde Q} + \|\delta_k^{\omega,sto}(s)\|_*)^2  + \|\delta_k^{\omega,det}(s)\|_* \bar{D}_{\cA, \|\cdot\|}.\nn\\
 &\le \eta_k (M_h + M_Q + M_{\tilde Q} +  \bar \sigma^{\omega})^2  + \bar \varsigma^\omega \bar{D}_{\cA, \|\cdot\|}. \label{eq:strong_ass_pmd_nonconvex3}
\end{align}
Similar to \eqnok{eq:PDA_bnd_Adv1}, we can show that
\begin{align}
V^{\pi_{k+1}}(s) - V^{\pi_k}(s)
&= \tfrac{1}{1-\gamma} \int \psi^{\pi_k}(q, \pi_{k+1}(q))]  \kappa_s^{\pi_{k+1}}(dq) \nn\\
&= \tfrac{1}{1-\gamma} \int \psi^{\pi_k}(q, \pi_{k+1}(q)) -
\left[\eta_k (M_h + M_Q + M_{\tilde Q} +\bar \sigma^{\omega})^2  +  \bar \varsigma^\omega \bar{D}_{\cA, \|\cdot\|} \right] \kappa_s^{\pi_{k+1}}(dq) \nn\\
&\quad + \tfrac{1}{1-\gamma}\left[\eta_k (M_h + M_Q + M_{\tilde Q} +\bar \sigma^{\omega})^2  + \bar \varsigma^\omega \bar{D}_{\cA, \|\cdot\|} \right] \nn\\
&\le \psi^{\pi_k}(s, \pi_{k+1}(s)) +
\tfrac{\gamma}{1-\gamma}\left[\eta_k (M_h + M_Q + M_{\tilde Q} +\bar \sigma^{\omega})^2  +  \bar \varsigma^\omega \bar{D}_{\cA, \|\cdot\|} \right]. \label{eq:strong_ass_pmd_nonconvex4}
\end{align}
The result in \eqnok{eq:stationary_SPMD0} follows by taking the telescopic sum of \eqnok{eq:strong_ass_pmd_nonconvex4},
while the one in \eqnok{eq:stationary_SPMD} follows directly from \eqnok{eq:stationary_SPMD0} and \eqnok{eq:strong_ass_pmd_nonconvex3}.
\end{proof}

\vgap

Observe that the assumptions in \eqnok{eq:strong_ass_pmd_nonconvex} 
are needed to provide a constant upper bound in \eqnok{eq:strong_ass_pmd_nonconvex3}.
Otherwise, if we only make the weaker assumptions in \eqnok{eq:weak_ass_pmd_convex2} and \eqnok{eq:weak_ass_pmd_convex3},
the relation in \eqnok{eq:strong_ass_pmd_nonconvex4} does not necessarily hold since
the distribution $\kappa_s^{\pi_{k+1}}$ depends on the random sample $\xi_k$ generated in the $k$-th iteration.
In the next section, we will show that this issue can possibly be addressed by using the policy dual averaging method.

\section{Policy Dual Averaging} \label{sec:exact_PDA}
Inspired by the dual averaging method for convex optimization~\cite{Nesterov2009,Xiao10-1}, we present in this section a new class of policy gradient type method, referred to as policy dual averaging (PDA), for reinforcement learning.
Similar to the previous section, we will first present the generic algorithmic scheme in~\cref{sec:gen_pda} and discuss the incorporation of function approximation into PDA in~\cref{sec:fun_app_PDA}.

\subsection{The Generic Algorithmic Scheme}~\label{sec:gen_pda}
In this section, we assume access to the exact action-value function $Q^\pi$ for a given policy $\pi$.
We present a basic policy dual averaging method
for RL and establish its convergence properties.

Let $\|\cdot\|$ be a given norm in action space $\cA$
and $\omega: \cA \to \bbr$ be a strongly convex function
w.r.t. the given $\|\cdot\|$, and let $D(\cdot, \cdot)$ be its associated Bregman distance defined in \eqnok{eq:omega_strong_convexity}. 
At the $k$th iteration of the PDA method (see Algorithm~\ref{basic_pda}), we minimize
 $\Psi_k$ defined as the weighted sum of the advantage functions plus the regularization term 
 $ \lambda_k  D(\pi_0(s), a)$
  (see \eqnok{eq:PDA_step}).
 For notational convenience, we set $\Psi_{-1}(s, a) = \lambda_{-1} D(\pi_0(s), a)$ for some $\lambda_{-1} \ge 0$.
 Since $\psi^{\pi_{t}}(s, a)$ differs from $Q^{\pi_{t}}(s, a) + h^a(s)$ by a constant independent 
of $a$, the update in \eqnok{eq:PDA_step} is equivalent to \eqnok{eq:PDA_step1},
and we only need to evaluate the action-value function $Q^{\pi_{k}}(s, a)$ at each iteration.
Different from the PMD method in the previous section, the prox-center $\pi_0(s)$
in the regularization term $ \lambda_k D(\pi_0(s), a)$ does not change over different iterations.
This fact makes it more convenient to implement function approximation techniques in PDA.
While it is easy to see that PDA is equivalent to PMD under certain specific settings (e.g. in the tabular setting
with KL divergence as the Bregman distance),
the algorithmic scheme of PDA is different from PMD in general, and its convergence
for solving general RL problems has not been studied before in the literature.

\begin{algorithm}[t]
\caption{The policy dual averaging (PDA) method}
\begin{algorithmic}
\State {\bf Input:} $\beta_k \ge 0$, $\lambda_k \ge 0$ and initial policy $\pi_0(s) \in \cA$.
\For {$k =0, 1, \ldots,$}
\begin{align} 
\pi_{k+1}(s) &= \argmin_{a \in \cA} \left\{ \Psi_k(s,a) := \tsum_{t=0}^k \beta_t \psi^{\pi_{t}}(s, a) + \lambda_k  D(\pi_0(s), a) \right\} \label{eq:PDA_step}\\
&= \argmin_{a \in \cA} \left \{ \tsum_{t=0}^k \beta_t [Q^{\pi_{t}}(s, a) + h^a(s)] + \lambda_k  D(\pi_0(s), a)\right\}, \label{eq:PDA_step1}
\forall s \in \cS.
\end{align}
\EndFor
\end{algorithmic} \label{basic_pda}
\end{algorithm}

In the remaining part of this subsection, our goal is to show that
PDA achieves mostly comparable performance to the PMD method
under exact policy evaluation. However,
the analysis of PDA is significantly different from
PMD as well as the dual averaging method for convex optimization.

Similar to the previous section, we need to assume that $Q^\pi(s, \cdot)$ is $\mu_Q$-weakly convex
and define $\mu_d := \mu_h-\mu_Q$ (see \eqnok{eq:def_convexity_PDA}).
We say that the advantage function $\psi^{\pi}(s, \cdot)$ is convex 
if $\mu_d \ge 0$, and weakly convex otherwise. Then,
for some properly chosen $\lambda_k \ge 0$, the subproblem in \eqnok{eq:PDA_step}
is strongly convex and has a unique solution characterized by 
the following result. 

\begin{lemma} \label{lem:three_point}
Suppose that $\{\beta_t\}$ and $\lambda_k$ are chosen such that
\beq \label{eq:stepsizerul}
\mu_k :=  \mu_d \tsum_{t=0}^k \beta_t  + \lambda_k \ge 0, \forall k \ge -1,
\eeq
where we denote $\mu_{-1} = \lambda_{-1}$.
Then for any $k \ge -1$, we have
\beq \label{eq:three_point}
\Psi_k(s,\pi_{k+1}(s)) + \mu_k D(\pi_{k+1}(s), a)
\le \Psi_k(s,a), \forall s \in \cS, \forall a \in \cA. 
\eeq
\end{lemma}

\begin{proof}
Note that the result clearly holds for $k = -1$.
For any $k \ge 0$,
observe that by \eqnok{eq:convexity_Q}, 
 $\Psi_k(s, \cdot)$ is strongly convex w.r.t. modulus $\mu_k$.
The result then follows immediately from the optimality condition of \eqnok{eq:PDA_step}~\cite[Lemma 3.5]{LanBook2020}.
\end{proof}

\vgap

\cref{lem:fun_decrease} establishes some descent property of the PDA method.

\begin{proposition} \label{lem:fun_decrease}
If $\beta_k > 0$, relation \eqnok{eq:stepsizerul} holds, 
and
\beq \label{eq:stepsizerul_inc}
\lambda_k \ge \lambda_{k-1}, \forall k \ge 0,
\eeq
then for any $k \geq 0$,
\begin{align} 
&V^{\pi_{k+1}}(s) - V^{\pi_{k}}(s)-  \tfrac{\gamma(\lambda_{k} - \lambda_{k-1})}{(1-\gamma)\beta_k} \edits{\bar{D}_0} \le \psi^{\pi_{k}}(s, \pi_{k+1}(s))\nn \\
 &\le - \tfrac{\mu_k}{\beta_k} \dist(\pi_{k+1}(s), \pi_{k}(s)) - \tfrac{\mu_{k-1}}{\beta_k} \dist(\pi_{k}(s), \pi_{k+1}(s)) +  \tfrac{ \lambda_{k} - \lambda_{k-1}}{\beta_k}\edits{\bar{D}_0},
\label{eq:fun_decrease}
\end{align}
where
\begin{align} \label{eq:barD_defn}
    \bar{D}_0 := \sup_{s \in \cS, a \in \cA} D(\pi_0(s), a).
\end{align}
\end{proposition}

\begin{proof}
First, we note that by the definition of $\Psi_k$ in \eqnok{eq:PDA_step}
and~\cref{lem:three_point} (applied to the $k-1$ iteration),
we have for any $k \ge 0$,
\begin{align}
\beta_k \psi^{\pi_{k}}(s, a) &= \Psi_k(s, a) - \Psi_{k-1}(s, a) - (\lambda_{k} - \lambda_{k-1}) D(\pi_0(s), a) \nn \\
&\le \Psi_k(s, a) - \Psi_{k-1}(s, \pi_{k}(s)) - \mu_{k-1} D(\pi_{k}(s), a) - (\lambda_k - \lambda_{k-1}) D(\pi_0(s), a). \label{eq:potential_diff}
\end{align}
Setting $a = \pi_{k+1}(s)$ and using \eqnok{eq:stepsizerul_inc} in the above relation, we obtain 
\begin{align}
\beta_k \psi^{\pi_{k}}(s, \pi_{k+1}(s)) &\le \Psi_k(s, \pi_{k+1}(s)) - \Psi_{k-1}(s, \pi_{k}(s)) - \mu_{k-1} D(\pi_{k}(s), \pi_{k+1}(s)).\label{eq:potential_diff1}  
\end{align}
By the above inequality, and~\cref{lem:three_point} 
(with $a = \pi_{k}(s)$ in \eqnok{eq:three_point}), 
we have
\begin{align}
&\beta_k \psi^{\pi_{k}}(s, \pi_{k+1}(s)) \nn \\
&\le \Psi_k(s, \pi_{k}(s)) - \mu_k D(\pi_{k+1}(s), \pi_{k}(s))  - \Psi_{k-1}(s, \pi_{k}(s)) - \mu_{k-1} \dist(\pi_{k}(s), \pi_{k+1}(s))\nn\\
&= \beta_k \psi^{\pi_{k}}(s, \pi_{k}(s)) - \mu_k \dist(\pi_{k+1}(s), \pi_{k}(s)) - \mu_{k-1} \dist(\pi_{k}(s), \pi_{k+1}(s))
+ (\lambda_{k} - \lambda_{k-1}) D(\pi_0(s), \pi_{k}(s)) \nn\\
&\leq - \mu_k \dist(\pi_{k+1}(s), \pi_{k}(s)) - \mu_{k-1} \dist(\pi_{k}(s), \pi_{k+1}(s))+ (\lambda_{k} - \lambda_{k-1}) \bar{D}_0 ,\label{eq:PDA_bnd_Adv}
\end{align}
where the last two relations follow from the definition of $\Psi_k$, the fact that $\psi^{\pi_{k}}(s, \pi_{k}(s)) = 0$, and~\eqnok{eq:stepsizerul_inc} together with~\eqnok{eq:barD_defn}.
The previous conclusion and the assumption $\beta_k > 0$ then clearly imply that
\beq \label{eq:PDA_bnd_Adv1}
\psi^{\pi_k}(s, \pi_{k+1}(s)) -\tfrac{1}{\beta_k} (\lambda_{k} - \lambda_{k-1})\bar{D}_0 
\le - \tfrac{\mu_k}{\beta_k} \dist(\pi_{k+1}(s), \pi_{k}(s)) - \tfrac{\mu_{k-1}}{\beta_k} \dist(\pi_{k}(s), \pi_{k+1}(s)) \le 0, \ \forall s \in \cS.
\eeq
By~\cref{lem:performance_diff_deter}, \eqnok{eq:PDA_bnd_Adv1}, and the fact that $\kappa_s^{\pi_{k+1}} (\{s\}) \ge 1-\gamma$ due to \eqnok{eq:def_visit_measure}, we have 
\begin{align}
V^{\pi_{k+1}}(s) - V^{\pi_{k}}(s)
&= \tfrac{1}{1-\gamma} \int[ \psi^{\pi_{k}}(q, \pi_{k+1}(q)) - \tfrac{1}{\beta_k} (\lambda_{k} - \lambda_{k-1}) \bar{D}_0]\kappa_s^{\pi_{k+1}} (d q) + \tfrac{1}{(1-\gamma)\beta_k}  (\lambda_{k} - \lambda_{k-1}) \bar{D}_0 \nn\\
&\le \tfrac{1}{1-\gamma}[ \psi^{\pi_{k}}(s, \pi_{k+1}(s)) - \tfrac{1}{\beta_k} (\lambda_{k} - \lambda_{k-1}) \bar{D}_0]  \kappa_s^{\pi_{k+1}} (\{s\}) +  \tfrac{1}{(1-\gamma)\beta_k} (\lambda_{k} - \lambda_{k-1}) \bar{D}_0 \nn \\
&\le \psi^{\pi_{k}}(s, \pi_{k+1}(s)) + \tfrac{\gamma}{(1-\gamma)\beta_k} (\lambda_{k} - \lambda_{k-1}) \bar{D}_0, \label{eq:PDA_bnd_Adv2}
\end{align}
which together with \eqnok{eq:PDA_bnd_Adv1} clearly implies the result.
\end{proof}

\vgap

Below we establish the linear convergence rate of PDA for the case when the advantage function 
defined in \eqnok{eq:def_advantage} is convex w.r.t. $a$, i.e.,
$\mu_d \ge 0$ in \eqnok{eq:stepsizerul}.

\begin{theorem} \label{the:PDA_basic}
 Suppose that $\beta_k > 0$ and that relations \eqnok{eq:stepsizerul} and \eqnok{eq:stepsizerul_inc} hold
for the PDA method. Then 
\begin{align}
& \tsum_{t=0}^{k-1}\left\{ \beta_t 
\left\{[f(\pi_{t+1}) - f(\pi^*) ] -\gamma [f(\pi_{t}) -f(\pi^*)  ]\right\}-  \tfrac{\gamma(\lambda_{t} - \lambda_{t-1})}{1-\gamma} \edits{\bar{D}_0} \right\}
+ \mu_{k-1} {\cD}(\pi_{k}, \pi^*) \nn \\
&+ \tsum_{t=0}^{k-1} [\mu_{t-1} {\cD}(\pi_{t}, \pi_{t+1})] \le \lambda_k \cD(\pi_0, \pi^*), \label{eq:pda_main_general}
\end{align}
for any $k \ge 1$, 
where ${\cD}$ and $\bar{D}_0$ are defined in~\eqnok{eq:cD_defn} and~\eqnok{eq:barD_defn}, respectively.
In particular,  if $\mu_d \ge 0$, $\beta_k = \gamma^{-k}$, and $\lambda_k = \lambda$ for some $\lambda > 0$
for any $k \ge 1$, then
\begin{align}
& [f(\pi_{k}) - f(\pi^*) ] + [(1-\gamma^{k})\mu_d / (1-\gamma) + \lambda \gamma^{k-1}] \cD(\pi_{k}, \pi^*) \nn \\
& \le \gamma^{k-1} \lambda \cD(\pi_0, \pi^*) + \gamma^k [f(\pi_0) -f(\pi^*)  ], \ \forall k \ge 1. \label{eq:pda_linear_rate}
\end{align}
If in addition, $\lambda = 0$ in the
PDA algorithm, then 
\[
[f(\pi_{k}) - f(\pi^*) ] + (1-\gamma^{k})\mu_d/(1-\gamma)  {\cD}(\pi_{k}, \pi^*) 
 \le \gamma^{k} [f(\pi_0) -f(\pi^*)  ], \ \forall k \ge 1.
\]
\end{theorem}

\begin{proof}
We first show~\eqnok{eq:pda_main_general}. 
By taking the telescopic sum of \eqnok{eq:potential_diff1}, we have
\begin{align*}
\tsum_{t=0}^{k-1} \beta_t \psi^{\pi_{t}}(s, \pi_{t+1}(s)) 
&\le \Psi_{k-1}(s, \pi_{k}(s)) - \Psi_{-1}(s, \pi_0(s)) -\tsum_{t=0}^{k-1}[ \mu_{t-1} \dist(\pi_{t}(s), \pi_{t+1}(s)) ] \\
&= \Psi_{k-1}(s, \pi_{k}(s))  -\tsum_{t=0}^{k-1} [\mu_{t-1} \dist(\pi_{t}(s), \pi_{t+1}(s)) ],
\end{align*}
where the equality follows from the fact that
$ \Psi_{-1}(s, \pi_0(s)) = \lambda_{-1} D(\pi_0(s), \pi_0(s)) = 0$.
It then follows from the above conclusion and~\cref{lem:three_point} that
\begin{align}
&\tsum_{t=0}^{k-1} \beta_t  \psi^{\pi_{t}}(s, \pi_{t+1}(s)) \nn \\
&\le \Psi_{k-1}(s, a) - \mu_{k-1} \dist(\pi_{k}(s), a) -\tsum_{t=0}^{k-1} [\mu_{t-1} \dist(\pi_{t}(s), \pi_{t+1}(s)) ] \nn \\
&= \tsum_{t=0}^{k-1} \beta_t \psi^{\pi_{t}}(s, a) + \lambda_{k-1} \, \dist(\pi_0(s),a) 
- \mu_{k-1} \dist(\pi_{k}(s), a) 
 -\tsum_{t=0}^{k-1} [\mu_{t-1} \dist(\pi_{t}(s), \pi_{t+1}(s)) ], \label{eq:PDA_common_rel}
\end{align}
where the identity follows from the definition of $\Psi_k$.
Using \eqnok{eq:fun_decrease} and the above inequality, we obtain
\begin{align}
&\tsum_{t=0}^{k-1}\beta_t [V^{\pi_{t+1}}(s) - V^{\pi_{t}}(s)-  \tfrac{\gamma(\lambda_{t} - \lambda_{t-1})}{(1-\gamma)\beta_t} \bar{D}_0] 
+\tsum_{t=0}^{k-1}[ \mu_{t-1} \dist(\pi_{t}(s), \pi_{t+1}(s)) ] \nn\\
&\le \tsum_{t=0}^{k-1} \beta_t \psi^{\pi_{t}}(s, a) + \lambda_{k-1} \, D(\pi_0(s),a)  - \mu_{k-1} \dist(\pi_{k}(s), a), \label{eq:PDA_main_rec}
\end{align}
for any $s \in \cS$ and $a \in \cA$. 
Setting $a = \pi^*(s)$ in \eqnok{eq:PDA_main_rec}, taking expectation w.r.t. $\nu^*$ on both sides,
and using~\cref{lem:monotonicity},
we then obtain
\begin{align*}
& \tsum_{t=0}^{k-1}\left\{ \int \beta_t [V^{\pi_{t+1}}(s) - V^{\pi_{t}}(s)] \nu^*(ds) - \tfrac{\gamma(\lambda_{t} - \lambda_{t-1})}{1-\gamma} \bar{D}_0 \right\}
+ \tsum_{t=0}^{k-1} \int  [\mu_{t-1} \dist(\pi_{t}(s), \pi_{t+1}(s))]\nu^*(ds) \\
&\le \tsum_{t=0}^{k-1} \beta_t (1-\gamma)\int [V^{\pi^*}(s) - V^{\pi_{t}}(s) ] \nu^*(d s) 
+ \textstyle \int [\lambda_{k-1} D(\pi_0(s),\pi^*(s)) - \mu_{k-1} \dist(\pi_{k}(s), \pi^*(s))] \nu^*(ds). 
\end{align*}
Using the definition of $f$ and ${\cal D}$, and rearranging the terms, we have
\begin{align*}
& \tsum_{t=0}^{k-1} \left\{\beta_t  [f(\pi_{t+1}) - f(\pi_{t})] -  \tfrac{\gamma(\lambda_{t} - \lambda_{t-1})}{1-\gamma} \bar{D}_0 \right\}+ \tsum_{t=0}^{k-1} [\mu_{t-1} {\cal D}(\pi_{t}, \pi_{t+1})]\\
&\le \tsum_{t=0}^{k-1} \beta_t (1-\gamma) [f(\pi^*) -f(\pi_{t})  ]  + \lambda_{k-1} {\cal D}(\pi_0, \pi^*) 
- \mu_{k-1} {\cal D}(\pi_{k}, \pi^*),
\end{align*}
which implies \eqnok{eq:pda_main_general} after rearranging the terms.
Now using \eqnok{eq:pda_main_general} and the selection of $\beta_k$ and $\lambda_k$,
we have
\[
\gamma^{-(k-1)}
 [f(\pi_{k}) - f(\pi^*) ] -\gamma [f(\pi_0) -f(\pi^*)  ]
+ \mu_{k-1} {\cal D}(\pi_{k}, \pi^*) \le \lambda {\cal D}(\pi_0, \pi^*).
\]
Also note that by \eqnok{eq:stepsizerul},
$
\mu_k = \mu_d \tsum_{t=0}^{k-1} \gamma^{-t} + \lambda = \tfrac{ \mu_d(1-\gamma^{k})} {\gamma^{k} (1-\gamma)} + \lambda.
$
Combining the above two relations, 
we obtain
\begin{align*}
& \gamma^{-(k-1)}
 [f(\pi_{k}) - f(\pi^*) ] + [(1-\gamma^{k}) \gamma^{-k} \mu_d/(1-\gamma) + \lambda] D(\pi_{k}, \pi^*) \\
& \le  \lambda D(\pi_0, \pi^*) + \gamma [f(\pi_0) -f(\pi^*)  ] ,
\end{align*}
from which \eqnok{eq:pda_linear_rate} follows immediately.
\end{proof}

\vgap

We add some remarks about~\cref{the:PDA_basic}.
First, the PDA method achieves a linear rate of convergence with properly chosen $\beta_k$.
Since $\beta_k$ in~\cref{the:PDA_basic} is geometrically increasing, in the limiting case
the PDA will be approaching the classic policy iteration method. Below we provide
a simple and alternative convergence analysis of the policy iteration method
as it might be of some interest in its own right.

\begin{corollary}
Suppose $\pi_{k+1}(s) = \argmin_{a \in \cA} \psi^{\pi_{k}}(s,a)$ for any $s \in \cS$.
Then we have 
$
f(\pi_{k+1}) - f(\pi^*) \le \gamma [f(\pi_{k}) - f(\pi^*)].
$
\end{corollary}

\begin{proof}
By the definition of $\pi_{k+1}$, we have
$\psi^{\pi_{k}}(s,\pi_{k+1}(s)) - \psi^{\pi_{k}}(s,a) \le 0$ for any $a \in \cA$.
Setting $a = \pi_{k}(s)$, we have 
$
\psi^{\pi_{k}}(s,\pi_{k+1}(s)) \le  \psi^{\pi_{k}}(s,\pi_{k}(s)) = 0.
$
Therefore, similarly to \eqnok{eq:PDA_bnd_Adv2}, we have
$
V^{\pi_{k+1}}(s) - V^{\pi_{k}}(s)
\le  \psi^{\pi_{k}}(s, \pi_{k+1}(s)). 
$
Using the first inequality established in the proof (with $a = \pi^*(s)$),
we have
$
V^{\pi_{k+1}}(s) - V^{\pi_{k}}(s) \le \psi^{\pi_{k}}(s,\pi^*(s)).
$
Taking expectation w.r.t.~$s \sim \nu^*$, using~\cref{lem:monotonicity} and the definition of $f$, we conclude 
$
f(\pi_{k+1}) - f(\pi_{k})] \le (1-\gamma)[f(\pi^*) - f(\pi_{k})],
$
from which the result immediately follows.
\end{proof}

\vgap

Second, in view of~\cref{the:PDA_basic}, one can set $\lambda = 0$ and so
it appears that adding the strongly convex term $\lambda_k D(\pi_0(s), a)$ in the definition of $\Psi_k$ is not needed.
It should be noted, however, that this term will be critical in the following sense:
a) for the convex settings (i.e., $\mu_d \ge 0$), this term will help with 
suppressing the impact of the noisy estimation of $Q^{\pi_k}$ as will be demonstrated later; and
b) for the nonconvex setting (i.e., $\mu_d < 0$), adding $\lambda_k D(\pi_0(s), a)$ will make the subproblem tractable and thus the policy is easily computable.  
The following result discusses the convergence of PDA for RL problems in the nonconvex setting.

\begin{theorem} \label{the:pda_nonconvex}
Suppose that $\mu_d < 0$.
Assume that the number of iterations $k$
is fixed a priori, $\lambda_t = \lambda = k(k+1) |\mu_d|$ and $\beta_t = t+1$,  $t = 0, \ldots, k-1$.
Then for any $s \in \cS$, there exist iteration indices
$\bar k(s)$ s.t.
\begin{align}
0\le - \psi^{\pi_{\bar k(s)}}&(s, \pi_{\bar k(s)+1}(s)) \le \tfrac{2[V^{\pi_{0}}(s) - V^{\pi_{*}}(s)]}{(1-\gamma) (k+1)}, \label{eq:PDA_nonconvex1}\\
\tfrac{\mu_{\bar k(s)}}{\beta_{\bar k(s)}} \dist(\pi_{\bar k(s)+1}(s), \pi_{\bar k(s)}(s)) +& \tfrac{\mu_{\bar k(s)-1}}{\beta_{\bar k(s)}} \dist(\pi_{\bar k(s)}(s), \pi_{\bar k(s)+1}(s)) 
\le \tfrac{2[V^{\pi_{0}}(s) - V^{\pi_{*}}(s)]}{(1-\gamma) (k+1)},\label{eq:PDA_nonconvex2}
\end{align}
where $\mu_k = |\mu_d| k(k+1)/2$.
\end{theorem}

\begin{proof}
Clearly, the relation \eqnok{eq:stepsizerul} is satisfied. It follows from \eqnok{eq:fun_decrease} that
\begin{equation}
\tfrac{\beta_k}{1-\gamma} [V^{\pi_{k+1}}(s) - V^{\pi_{k}}(s)] \le \beta_k \psi^{\pi_{k}}(s, \pi_{k+1}(s))
\le 0. \label{eq:fun_dec_pda}
\end{equation}
Taking the telescopic sum, we have
$
0\le \tsum_{t=0}^{k-1} [-\beta_t \psi^{\pi_{t}}(s, \pi_{t+1}(s))]  
\le \tfrac{1}{1-\gamma}\tsum_{t=0}^{k-1} \beta_t [V^{\pi_{t}}(s) - V^{\pi_{t+1}}(s)].
$
The indices $\bar k(s)$ are given by the one that minimizes $-\beta_t \psi^{\pi_{t}}(s, \pi_{t+1}(s))$ for any $s \in \cS$.
Also notice that 
\begin{align}
\tsum_{t=0}^{k-1} \beta_t[V^{\pi_{t}}(s) - V^{\pi_{t+1}}(s)] &=  \beta_0 V^{\pi_0(s)} + \tsum_{t=0}^{k-2}(\beta_{t+1} - \beta_{t}) V^{\pi_{t+1}}(s) -\beta_{k-1} V^{\pi_k}(s) \nn \\
&=  \tsum_{t=0}^{k-1} V^{\pi_t}(s) - kV^{\pi_k}(s) 
\le k[V^{\pi_0}(s) - V^{\pi^*}(s) ], \label{eq:bnd_sum_pda_value0}
\end{align}
where the first inequality follows from the monotonicity of $V^{\pi_t}$ due to \eqnok{eq:fun_dec_pda} 
and the second inequality holds since $V^{\pi_k}(s) \ge V^{\pi^*}(s)$.
Combining the above two relations, we obtain the result in \eqnok{eq:PDA_nonconvex1}.
The result in  \eqnok{eq:PDA_nonconvex2} follows immediately from \eqnok{eq:fun_decrease}
and \eqnok{eq:PDA_nonconvex1}.
\end{proof}

\vgap

Comparing~\cref{the:PDA_basic} and~\cref{the:pda_nonconvex} 
with~\cref{the:PMD_linear} and~\cref{the:pmd_nonconvex}, we can see that the
PDA method exhibits similar convergence behavior to the PMD method.
However, a possible limitation for the PDA method is that the relation in \eqnok{eq:PDA_nonconvex2} does not necessarily imply the stationarity of the policy $\pi_{\bar k(s)}(s)$ in the sense of~\eqref{eq:stationary_pt}, while the PMD method can possibly ensure stationarity as discussed after~\cref{the:pmd_nonconvex}.  

\subsection{Function Approximation in PDA} \label{sec:fun_app_PDA}
In this subsection, we discuss how to incorporate function approximation into the PDA method.
Similar to the PMD method, we also adopt the idea of using function approximation for the value functions.
However, different from the PMD method that requires function approximation
for an augmented value function $L^{\pi_k}(s,a)$, 
we will compute a function approximation
$\tilde Q(s, a; \theta_k)$, parameterized by $\theta_k$, only 
for the action-value function $Q^{\pi_k}(s, a)$ in \eqnok{eq:PDA_step1}, i.e.,
\beq \label{eq:policy_evaluation}
\tilde Q(s, a; \theta_k) \approx Q^{\pi_k}(s, a).
\eeq
As a result, the algorithmic parameters $\beta_k$ and $\lambda_k$ do not appear in the quantities of interest that we approximate. 
On the other hand, in PMD, the stepsize $\eta_k$ appears in $L^{\pi_k}$, which may impact the magnitude of the stochastic first-order information of the policy evaluation problem. It also required stronger assumptions in \eqnok{eq:strong_ass_pmd_nonconvex} to guarantee convergence of PMD method when $\tilde \mu_d < 0$.

\begin{algorithm}[H]
\caption{The policy dual averaging method with function approximation}
\begin{algorithmic}
\State {\bf Input:} $\beta_k \ge 0$, $\lambda_k \ge 0$ and initial policy $\pi_0(s) \in \cA$.
\For {$k =0,1,\ldots,$}

\State Define an approximate advantage function
\beq \label{eq:def_tpsi}
\tilde \psi(s, a; \theta_k) := \tilde Q(s, a; \theta_k) - \edits{\tilde Q(s, \pi_k(s); \theta_k) + h^a(s) - h^{\pi_k(s)}(s)}.
\eeq

\State Replace the update formulas in \eqnok{eq:PDA_step}-\eqnok{eq:PDA_step1} by
\begin{align} 
\pi_{k+1}(s) &= \argmin_{a \in \cA} \left\{ \tilde \Psi_k(s,a) := \tsum_{t=0}^k \beta_t \tilde \psi(s, a; \theta_{t}) + \lambda_k D(\pi_0(s), a) \right\} \label{eq:SPDA_step}\\
&= \argmin_{a \in \cA} \left \{ \tsum_{t=0}^k \beta_t [\tilde Q(s, a; \theta_{t}) + h^a(s)] + \lambda_k D(\pi_0(s), a)\right\}, \label{eq:SPDA_step1}
\forall s \in \cS.
\end{align}

\EndFor
\end{algorithmic} \label{alg:basic_inf_pda}
\end{algorithm}

Similar to PMD (c.f.~the discussion in~\cref{sec:func_approx_in_pmd}), we can consider linearly independent basis functions
or more powerful and sophisticated nonlinear function approximation methods, such as RKHS and
neural networks, to estimate $Q^{\pi_k}(s, a)$ via $\tilde Q(s,a;\theta_k)$. 

Also similar to PMD, PDA does not need to save $\pi_{k+1}(s)$ across all states $s$. Rather,~\eqnok{eq:SPDA_step1} tells us how to compute it from the parameters $\theta_{[k]}: =\{\theta_0, \ldots, \theta_{k}\}$ whenever needed, e.g., inside the policy evaluation procedure.
In fact, given $\theta_{[k]}$ and any $s \in \cS$, $\pi_{k+1}(s)$ can be computed by an explicit formula in some important cases. See~\eqnok{eq:PMD_step_finite} and the surrounding discussion for an example.
In a more general setting, $\pi_{k+1}(s)$ is represented as an optimal solution of problem \eqnok{eq:SPDA_step1}.
With a properly chosen $\lambda$, the objective function of \eqnok{eq:SPDA_step1} is strongly convex, and hence the policy $\pi_{k+1}(s)$ is uniquely defined. 
In many cases, the resulting~\eqnok{eq:SPDA_step1} can be solved efficiently using numerical optimization procedures, such as when the dimension of the action space is smaller than that of the state space.
The numerical efficiency can be further enhanced by observing the term
$
\tilde Q^k(s, a; \theta_{[k]}) := \tsum_{t=0}^k \beta_t \tilde Q(s, a; \theta_{t})
$
used in \eqnok{eq:SPDA_step1} admits the recursive formula 
\beq \label{eq:SPDA_recursive1}
\tilde Q^k(s, a; \theta_{[k]})
= \tilde Q^{k-1}(s, a; \theta_{[k-1]}) + \beta_k \tilde Q(s, a; \theta_{k}),
\eeq
with $\tilde Q^0(s, a; \theta_{[-1]}) = 0$.
When $\tilde Q(s, a; \theta)$ is linear w.r.t. $\theta$, the above formula further simplifies to a recursive updating on $\tsum_{t=0}^k \beta_t \theta_{t}$. 
In this way, we can also significantly save the memory required to store the parameters $\theta_{[t]}$.
For the more general case (e.g., $\tilde Q(s, a; \theta)$ is given by a neural network), in order to save both computation and memory, we can possibly modify the policy evaluation step in \eqnok{eq:policy_evaluation} so that we have a single set of parameters $\Theta_k$ satisfying
\beq \label{eq:SPDA_recursive2}
\tilde Q(s, a; \Theta_{k}) \approx \tilde Q(s, a; \Theta_{k-1}) +\beta_k Q^{\pi_{k}}(s, a)
\eeq
at the $k$-th iteration, with $\tilde Q(s, a; \Theta_{-1}) = 0$.
In other words, we would apply function approximation directly for estimating $\tsum_{t=0}^k \beta_t  Q^{\pi_{t}}(s, a)$
in a recursive manner. 
However, this would require us to design a modified policy evaluation procedure and to consider the impact of the algorithmic parameters $\beta_k$ for function approximation.
For the sake of simplicity, in this section we focus on the policy evaluation satisfying~\eqnok{eq:policy_evaluation} and assume exact solutions can be computed for \eqnok{eq:SPDA_step1}. 

\subsection{PDA for General State and Continuous/Finite Action Spaces}  \label{sec:continuous_action_PDA}
We now investigate the impact of the approximation error in~\eqnok{eq:policy_evaluation} on the convergence of the PDA method for solving RL problems. 
Observe that when estimating $Q^{\pi_{k}}$, there exist two possible types of stochasticity including: a) the sampling of $(s,a)$ from $\cS \times \cA$ according to some given initial distribution $\rho_\cS \times \rho_\cA$;
and b) the estimation of $Q^{\pi_{k}}(s, a)$ through the sampling of the Markov Chain induced by $\pi_{k}$. 
Similar to previous discussions, we use $\xi_k$ to denote these random samples generated at iteration $k$ and $F_k := \{\xi_1, \ldots, \xi_k\}$.
Let 
\beq \label{eq:def_app_error_pda}
\delta_{k}(s,a) := \tilde Q(s, a; \theta_{k}) - Q^{\pi_{k}}(s,a)
=
\underbrace{\mathbb{E}_{\xi_k}[\tilde Q(s, a; \theta_{k}) ] - Q^{\pi_k}(s,a)}_{\delta^{det}_k(s,a)} + \underbrace{\tilde Q(s, a; \theta_{k}) - \mathbb{E}_{\xi_k}[\tilde Q(s, a; \theta_{k}) ]}_{\delta^{sto}_k(s,a)}, \eeq
where $\bbe_{\xi_k}$ denotes the expectation w.r.t. $\xi_k$ conditioned on $F_{k-1}$.
We need to make the following assumption on these error terms:
\begin{align}
\bbe_{s \sim \nu^*}\left[|\delta_{k-1}^{det}(s, \pi_{k-1}(s))| + |\delta_{k-1}^{det}(s, \pi^*(s))|\right] \le \varsigma, \label{eq:bnd_error_spda}
\end{align}
for some $\varsigma \geq 0$.
Similar to~\eqnok{eq:weak_ass_pmd_convex1} for PMD, we do not need to explicitly bound the stochastic error term $\delta^{sto}_{k-1}$.
In addition, we need to assume that 
$
\tilde Q(s, a;\theta_k) + \tilde \mu_{Q} D(\pi_0(s), a)
$
is convex for some $\tilde \mu_Q \ge 0$ and denote
$
\tilde \mu_d := \mu_h - \tilde \mu_Q.
$
Note that $\tilde \mu_{Q}$ can be viewed as an estimation of $\mu_Q$.

Moreover, we assume that $\tilde Q(s, \cdot; \theta)$ and $Q^\pi(s, a)$ are Lipschitz continuous, i.e.,~\eqnok{eq:Lip_psi} and~\eqnok{eq:Lip_Q} still hold with constants $M_{\tilde Q}$ and $M_{Q}$, respectively.
We make a couple remarks.
First, note that if $\tilde Q$ is $M_{\tilde Q}$-Lipschitz continuous w.r.t.~$a$, then $\tilde \psi$ is $(M_{\tilde Q}\edits{+M_h})$-Lipschitz continuous.
Moreover, $\delta_t(s,a)$ is $(M_{\tilde Q}+ M_Q)$-Lipschitz continuous w.r.t. $a \in \cA$.
Second,~\cref{sec:appendix_pe} presents sufficient conditions for these assumptions to take place.

\begin{theorem} \label{thm:4-1}
Suppose $\tilde \mu_d \ge 0$ and that $\beta_k$ and $\lambda_k$ satisfy 
\beq \label{eq:stepsizerulSPDA}
\tilde \mu_k := \tilde \mu_d \tsum_{t=0}^k\beta_t + \lambda_k \ge 0 \ \
\mbox{and} \ \
\lambda_{k+1} \ge \lambda_{k}, \forall k \geq -1,
\eeq
where we denote $\tilde \mu_{-1} = \lambda_{-1}$.
Also let us denote $
\bar \beta_k : = \tsum_{t=0}^{k-1} \beta_t.
$ Then we have
\begin{align}
&(1-\gamma) {\bar \beta_k}^{-1} \tsum_{t=0}^{k-1} \beta_t \bbe[ f(\pi_{t}) - f(\pi^*)]  + { \bar \beta_k}^{-1} \tilde \mu_{k-1} \bbe[{\cal D}(\pi_{k}, \pi^*)] \nn\\
&\le \lambda_{k-1} {\bar \beta_k}^{-1} {\cal D}(\pi_0,\pi^*) + (M_{\tilde Q}^2 + M_h^2)  \bar \beta_k^{-1} \tsum_{t=0}^{k-1}(\beta_t^2   /\tilde \mu_{t-1} )
+ \varsigma.\label{eq:SPDA_main} 
\end{align}
where ${\cD}$ is defined in~\eqnok{eq:cD_defn}.
In particular, if $\tilde \mu_d > 0$, $\beta_k = k+1$ and $\lambda_k = \tilde \mu_d$, then 
\begin{align}
\tfrac{2(1-\gamma)}{k(k+1)} \tsum_{t=0}^{k-1} \{(t+1) \bbe[ f(\pi_{t-1}) - f(\pi^*)]\} + \tilde \mu_d \bbe[{\cal D}(\pi_{k}, \pi^*)]
\le \tfrac{2\tilde \mu_d  {\cal D}(\pi_0(s),\pi^*)} {k (k+1)} + \tfrac{4 (M_{\tilde Q}^2 + M_h^2)}{\tilde \mu_d (k+1)} + \varsigma. \label{eq:SPDA_main_strong} 
\end{align}
Moreover, if $\tilde \mu_d = 0$, $\beta_k = k+1$ and $\lambda_k = \lambda (k+1)^{3/2}$ with $\lambda = \sqrt{M_{\tilde Q}^2+M_h^2} / \sqrt{{\cal D}(\pi_0(s),\pi^*)}$, then
\begin{align}
\tfrac{2(1-\gamma)}{k(k+1)} \tsum_{t=0}^{k-1} \{t \bbe[ f(\pi_{t}) - f(\pi^*)]\} \le \tfrac{2 \lambda {\cal D}(\pi_0(s),\pi^*) \sqrt{k}  }{k+1}
+ \tfrac{4 (M_{\tilde Q}^2+M_h^2) \sqrt{k+1}}{\lambda k} + \varsigma.  \label{eq:SPDA_main_weak} 
\end{align}
\end{theorem}

\begin{proof}
Similar to \eqnok{eq:PDA_common_rel}, we can show that
\begin{align}
\tsum_{t=0}^{k-1} \beta_t  \tilde \psi(s, \pi_{t+1}(s), \theta_{t}) 
&\le \tsum_{t=0}^{k-1} \beta_t \tilde \psi(s, a, \theta_{t}) + \lambda_{k-1} \, \dist(\pi_0(s),a) - \tilde \mu_{k-1} \dist(\pi_{k}(s), a)  \nn\\
& \quad -\tsum_{t=0}^{k-1}\tilde \mu_{t-1} \dist(\pi_{t}(s), \pi_{t+1}(s)). \label{eq:spda_main_relation_convex}
\end{align}
Also note that $\tilde \psi(s, a; \theta_{t}) = \psi^{\pi_{t}}(s, a) + \delta_{t}(s, a)$ and \edits{$\tilde \psi(s,\cdot;\theta_k)$ is $(M_{\tilde Q}+M_h)$-Lipschitz continuous}, and so
\begin{align}
 &\tilde \psi(s, \pi_{t+1}(s), \theta_{t}) - \tilde{\psi}(s,a;\theta_t) \nn \\
 &=  
 -\psi^{\pi_t}(s,a) + \tilde \psi(s, \pi_{t+1}(s), \theta_{t}) -  \tilde \psi(s, \pi_{t}(s), \theta_{t}) +  \tilde \psi(s, \pi_{t}(s), \theta_{t}) \nn -\delta_t(s,a) \\
 &\ge -\psi^{\pi_t}(s,a)- (M_{\tilde Q}+M_h) \|\pi_{t+1}(s) - \pi_{t}(s)\| + \tilde \psi(s, \pi_{t}(s), \theta_{t}) -\delta_t(s,a) \nn \\
 &=  -\psi^{\pi_t}(s,a)- (M_{\tilde Q}+M_h) \|\pi_{t+1}(s) - \pi_{t}(s)\| + \psi^{\pi_{t}}(s, \pi_{t}(s)) + \delta_{t}(s, \pi_{t}(s)) -\delta_t(s,a) \nn \\
 &= -\psi^{\pi_t}(s,a)- (M_{\tilde Q}+M_h) \|\pi_{t+1}(s) - \pi_{t}(s)\| + \delta_{t}(s, \pi_{t}(s)) - \delta_t(s,a). \label{eq:pda_noisy_psi_bound}
\end{align}
Combining these relations and rearranging the terms, we obtain
\begin{align*}
&-\tsum_{t=0}^{k-1} \beta_t \psi^{\pi_{t}}(s, a) + \tilde \mu_{k-1} \dist(\pi_{k}(s), a) 
\\
&\le \lambda_{k-1} \, \dist(\pi_0(s),a) 
+ \tsum_{t=0}^{k-1}
\beta_t [\delta_{t}(s, a) -\delta_{t}(s, \pi_{t}(s)) ] \\
&\quad + \tsum_{t=0}^{k-1}\left[ \beta_t (M_{\tilde Q}+M_h) \|\pi_{t+1}(s) - \pi_{t}(s)\| - \tilde \mu_{t-1} \dist(\pi_{t}(s), \pi_{t+1}(s))\right]\\
&\le \lambda_{k-1} \, \dist(\pi_0(s),a) 
+ \tsum_{t=0}^{k-1}
\beta_t [(\delta_{t}^{det}(s, a) + \delta_{t}^{sto}(s, a)) -(\edits{\delta_{t}^{det}(s, \pi_{t}(s)) + \delta_{t}^{sto}(s, \pi_{t}(s))})] \\
&\quad +  (M_{\tilde Q}+M_h)^2\tsum_{t=0}^{k-1}(\beta_t^2 /(2\tilde \mu_{t-1})).
\end{align*}
\sloppy Setting $a = \pi^*(s)$, taking expectation over $s \sim \nu^*$ and then over $F_k$, noting~$\mathbb{E}_{F_k}[\delta^{sto}_t(s,\pi_t(s))] = \mathbb{E}_{F_k}[\delta^{sto}_t(s,\pi^*(s))] = 0$, and using the assumption in \eqnok{eq:bnd_error_spda} and~\cref{lem:monotonicity},
we then have
\begin{align*}
&(1-\gamma)  \tsum_{t=0}^{k-1} \beta_t \bbe[ f(\pi_{t}) - f(\pi^*)] + \tilde \mu_{k-1} \bbe[{\cal D}(\pi_{k}, \pi^*)] \\
&\le  \lambda_{k-1} \, {\cal D}(\pi_0(s),a) 
+ \tsum_{t=0}^{k-1} \beta_t \varsigma + (M_{\tilde Q}+M_h)^2\tsum_{t=0}^{k-1}(\beta_t^2 /(2\tilde \mu_{t-1})),
\\
&\le  \lambda_{k-1} \, {\cal D}(\pi_0(s),a) 
+ \tsum_{t=0}^{k-1} \beta_t \varsigma + (M_{\tilde Q}^2+M_h^2)\tsum_{t=0}^{k-1}(\beta_t^2 /\tilde \mu_{t-1}),
\end{align*}
where the last line is by $(p+q)^2 \leq 2p^2 + 2q^2$.
This clearly implies \eqnok{eq:SPDA_main}.
The result in \eqnok{eq:SPDA_main_strong}
follows from straightforward calculations, and \eqnok{eq:SPDA_main_weak} 
can be verified similarly in view of the observation that
$
 \tsum_{t=0}^{k-1}(\beta_t^2   /\tilde \mu_{t-1} )=
 \tfrac{1}{\lambda}\tsum_{t=0}^{k-1}\sqrt{t+1} \le   \tfrac{2}{3\lambda} (k+1)^{3/2}.
$
\end{proof}
\vgap
In comparison to the PMD method, PDA does not need to fix the number of iterations $k$ a priori when $\mu_h = 0$.
As a result, the stepsize selection of the PDA method is more adaptive than PMD for this case.

Similar to~\cref{cor:PMD_general_cvg}, we cannot generally have convenient knowledge of $\tilde \mu_d$ since it depends on the (unknown) lower curvature of $\tilde Q$. In the special case when $\cA$ is given as a finite set of actions, we can guarantee global convergence of PDA, similar to the PMD method in~\cref{thm:pmd_finite}. 
To setup notation, recall that we express $Q^{\pi_k}(s,a) = \langle Q^{\pi_k}_0(s), a \rangle$ and $\tilde Q(s,a;\theta_k) = \langle \tilde Q_0(s; \theta_k), a \rangle$ for some vectors $Q^{\pi_k}_0(s), \tilde{Q}_0(s; \theta_k) \in \mathbb{R}^{\nA}$ as described in~\eqref{eq:PMD_step_finite} and the surrounding discussion.
We write $\delta_{0,k}^{det}(s) := \mathbb{E}_{\xi_k}[\tilde Q_0(s;\theta_k)] - Q^{\pi_k}_0(s)$ and $\delta_{0,k}^{sto}(s) := \tilde{Q}_0(s;\theta_k) - \mathbb{E}_{\xi_k}[\tilde{Q}_0(s;\theta_k)]$, which are similar to the terms appearing in~\eqnok{eq:def_app_error_pda}.
In the following result, we will assume
\begin{align}
    \bbe_{s \sim \nu^*} [ \|\delta_{0,k}^{det}(s)\|_\infty] &\le \bar \varsigma,\label{eq:pda_finite_assum1}\\
    \bbe_{\xi_k}\left\{\bbe_{s \sim \nu^*}[\|\delta_{0,k}^{sto}(s)\|_\infty^2] \right\}&\le \bar{\sigma}^2, \label{eq:pda_finite_assum3}
\end{align}
for some $\bar{\varsigma} > 0$ and $\bar{\sigma} > 0$. These assumptions are identical to~\eqnok{eq:pmd_variance_bound_1}-\eqnok{eq:pmd_variance_bound_2} for PMD.
\begin{theorem} \label{thm:pda_finite}
    Suppose $\cA$ is given as a finite set of actions, and we extend it to continuous actions via randomized policies as detailed in~\eqnok{eq:finite_Q_function} and the surrounding discussion.
    If~\eqnok{eq:pda_finite_assum1}-\eqnok{eq:pda_finite_assum3} is satisfied, then the convergence results from~\cref{thm:4-1} hold with
    \begin{align*}
        \tilde \mu_d = \mu_h, \quad M_{\tilde{Q}} = \sqrt{2}[\tfrac{\bar{c}}{1-\gamma} + \bar{\sigma}], \quad \bar{D}_{\cA, \|\cdot\|} = 2, \quad\text{and}~ \varsigma = \bar{\varsigma} \bar{D}_{\cA, \|\cdot\|},
    \end{align*}
    where $\|\cdot\| = \|\cdot\|_1$.
\end{theorem}
\begin{proof}
    We modify~\eqnok{eq:pda_noisy_psi_bound} into
    \begin{align*}
        &\tilde{\psi}(s,\pi_{t+1}(s);\theta_t) - \tilde{\psi}(s,a;\theta_t)
        \\
        &= 
        -\psi^{\pi_t}(s,a)
        +
        \langle Q^{\pi_t}(s) + \delta_{0,t}^{sto}(s), \pi_{t+1}(s) - \pi_t(s) \rangle + [h^{\pi_{t+1}(s)}(s) - h^{\pi_t(s)}(s)] 
        \\
        &\hspace{10pt}
        + \langle \delta_{0,t}^{det}(s), \pi_{t+1}(s) - a \rangle + \langle \delta_{0,t}^{sto}(s), \pi_{t}(s) - a \rangle \\
        & \geq
        -\psi^{\pi_t}(s,a) - [\|Q^{\pi_t}(s)\|_\infty + \|\delta_{0,t}^{sto}(s)\|_\infty + M_h]\|\pi_{t+1}(s) - \pi_t(s)\| \\
        &\hspace{10pt}
        - \|\delta_{0,t}^{det}(s)\|_\infty \|\pi_{t+1}(s) - a \| + \langle \delta_{0,t}^{sto}(s), \pi_{t}(s) - a \rangle, \quad \forall a \in \cA.
    \end{align*}
    The rest of the proof follows similarly to the proof for~\cref{thm:4-1}.
\end{proof}

\vgap
We now consider the case when $\tilde \mu_d < 0$. 
Unlike the counterpart~\cref{the:PMD_fun_stationary} for PMD, here we only need to assume the Lipschitz continuity of $\tilde Q(s,a;\theta)$ and $Q^\pi(s,a)$
w.r.t.~$a \in \cA$ to establish a convergence guarantee.

\begin{theorem} \label{the:PDA_nonconvex}
Suppose that $\tilde \mu_d < 0$.
Assume that
the number of iterations $k$
is fixed a priori, $\lambda_t = \lambda = k(k+1) |\tilde \mu_d|$, and $\beta_t = t+1$,
$t=0, \ldots, k-1$. then for any $s \in \cS$, there exist iteration indices
$\bar k (s)$ s.t.
\begin{align}
-  \tfrac{(M_Q+M_{\tilde Q})^2}{|\tilde \mu_d|(k+1)} 
\le 
-\psi^{\pi_{\bar k{s}}}(s, \pi_{\bar k(s)+1}(s))
\le  
\tfrac{2[V^{\pi_0}(s)  - V^{\pi^*}(s) ]}{k+1}
 + \tfrac{3 (M_Q+M_{\tilde Q})^2}{(1-\gamma)|\tilde \mu_d| (k+1)}. \label{eq:PDA_nonconex_stoc1}
\end{align}
Moreover, we have
\begin{align}
 &\tfrac{\tilde \mu_{\bar k(s) }}{2\beta_{\bar k(s)}} \dist(\pi_{\bar k(s)+1}(s), \pi_{\bar k(s)}(s))+ \tfrac{\tilde \mu_{\bar k(s)-1}}{2\beta_{\bar k(s)}} \dist(\pi_{\bar k(s)}(s), \pi_{\bar k(s)+1}(s))\nn\\
 &\le \tfrac{2[V^{\pi_0}(s)  - V^{\pi^*}(s) ]}{k+1}
 + \tfrac{3 (M_Q+M_{\tilde Q})^2}{(1-\gamma)|\tilde \mu_d| (k+1)}+
 \tfrac{2 (M_Q+M_{\tilde Q})^2}{|\tilde \mu_d| (k+1)}, \label{eq:PDA_nonconex_stoc2}
\end{align}
where $\tilde \mu_t = [t(t+1) \tilde \mu_d/2+ k(k+1) |\tilde \mu_d|$, $t = 0, \ldots, k-1$.
\end{theorem}

\begin{proof}
Similar to \eqnok{eq:PDA_bnd_Adv1}, we can show that
\begin{align*}
&\tilde \psi(s, \pi_{k+1}(s), \theta_{k})
-\tfrac{1}{\beta_k} |\lambda_{k+1} - \lambda_{k}| D(\pi_0(s), \pi_k(s))
\\
&\le \tilde \psi(s, \pi_{k}(s), \theta_{k}) - \tfrac{\tilde \mu_k}{\beta_k} \dist(\pi_{k+1}(s), \pi_{k}(s)) - \tfrac{\tilde \mu_{k-1}}{\beta_k} \dist(\pi_{k}(s), \pi_{k+1}(s)),
\end{align*}
which in view of the definition of $\delta_k(s,a)$ in \eqnok{eq:def_app_error_pda} and $\psi^{\pi_{k}}(s, \pi_{k}(s)) =0$ then implies that
\begin{align*}
&\psi^{\pi_{k}}(s, \pi_{k+1}(s)) + \delta_{k}(s,\pi_{k+1}(s)) 
-\tfrac{1}{\beta_k} |\lambda_{k+1} - \lambda_{k}| D(\pi_0(s), \pi_k(s)) \\
&\le   \delta_{k}(s,\pi_{k}(s))  - \tfrac{\tilde \mu_k}{\beta_k} \dist(\pi_{k+1}(s), \pi_{k}(s)) - \tfrac{\tilde \mu_{k-1}}{\beta_k} \dist(\pi_{k}(s), \pi_{k+1}(s))\\
&\le \delta_{k}(s,\pi_{k}(s)) - \tfrac{1}{4 \beta_k}(\tilde \mu_k + \tilde \mu_{k-1}) \|\pi_{k+1}(s) - \pi_{k}(s)\|^2 \\
&\quad - \tfrac{\tilde \mu_k}{2\beta_k} \dist(\pi_{k+1}(s), \pi_{k}(s))- \tfrac{\tilde \mu_{k-1}}{2\beta_k} \dist(\pi_{k}(s), \pi_{k+1}(s)).
\end{align*}
Also note that
\begin{align*}
& \delta_{k}(s,\pi_{k}(s)) -\delta_{k}(s,\pi_{k+1}(s))  - \tfrac{1}{4 \beta_k}(\tilde \mu_k + \tilde \mu_{k-1}) \|\pi_{k+1}(s) - \pi_{k}(s)\|^2\\
&\le (M_Q+M_{\tilde Q}) \|\pi_{k+1}(s) -\pi_{k}(s)\| - \tfrac{1}{4 \beta_k}(\tilde \mu_k + \tilde \mu_{k-1}) \|\pi_{k+1}(s) - \pi_{k}(s)\|^2\\
&\le \tfrac{\beta_k (M_Q+M_{\tilde Q})^2}{\tilde \mu_k + \tilde \mu_{k-1}}.
\end{align*}
Combining these two relations, we have
\begin{align}
\psi^{\pi_{k}}(s, \pi_{k+1}(s)) + \tfrac{\tilde \mu_k}{2\beta_k} \dist(\pi_{k+1}(s), \pi_{k}(s))+ \tfrac{\tilde \mu_{k-1}}{2\beta_k} \dist(\pi_{k}(s), \pi_{k+1}(s)) \nn
\le 
\tfrac{|\lambda_{k+1} - \lambda_{k}| D(\pi_0(s), \pi_k(s))}{\beta_k}  + \tfrac{\beta_k (M_Q+M_{\tilde Q})^2}{\tilde \mu_k + \tilde \mu_{k-1}}. \label{eq:bnd_pda_dist_stationary}
\end{align}
Using the above relation, we can show similarly to \eqnok{eq:PDA_bnd_Adv2} that
\begin{align}
V^{\pi_{k+1}}(s) - V^{\pi_{k}}(s)
&\le \psi^{\pi_{k}}(s, \pi_{k+1}(s)) + \tfrac{\gamma |\lambda_{k+1} - \lambda_{k}| D(\pi_0(s), \pi_k(s))}{(1-\gamma)\beta_k}  
+\tfrac{\gamma \beta_k (M_Q+M_{\tilde Q})^2}{(1-\gamma)(\tilde \mu_k + \tilde \mu_{k-1})}. \label{eq:bnd_pda_dist_stationary1}
\end{align}
Using the fact that $\lambda_k = \lambda_{k-1}$ in the previous two relations, we conclude that
\begin{align*}
\beta_k[V^{\pi_{k+1}}(s) - V^{\pi_{k}}(s)] - \tfrac{\gamma \beta_k^2 (M_Q+M_{\tilde Q})^2}{(1-\gamma)(\tilde \mu_k + \tilde \mu_{k-1})} 
\le \beta_k \psi^{\pi_{k}}(s, \pi_{k+1}(s)) \le \tfrac{\beta_k^2 (M_Q+M_{\tilde Q})^2}{\tilde \mu_k + \tilde \mu_{k-1}}. 
\end{align*}
Rearranging terms in the above inequality, we then have
\begin{align}
0 \le -\beta_k \big[ \psi^{\pi_{k}}(s, \pi_{k+1}(s)) - \tfrac{\beta_k (M_Q+M_{\tilde Q})^2}{\tilde \mu_k + \tilde \mu_{k-1}}\big]
\le \beta_k[V^{\pi_{k}}(s) - V^{\pi_{k+1}}(s) ] + \tfrac{ \beta_k^2 (M_Q+M_{\tilde Q})^2}{(1-\gamma)(\tilde \mu_k + \tilde \mu_{k-1})}. \label{eq:spda_nonconvex}
\end{align}
Moreover, summing up \eqnok{eq:bnd_pda_dist_stationary} and \eqnok{eq:bnd_pda_dist_stationary1}, we have
$
V^{\pi_{k+1}}(s) - V^{\pi_{k}}(s) \le \tfrac{ \beta_k (M_Q+M_{\tilde Q})^2}{(1-\gamma)(\tilde \mu_k + \tilde \mu_{k-1})}$,
which implies that
\[
V^{\pi_{k}}(s) \le V^{\pi_{0}}(s) + \tsum_{t=0}^{k-1} \tfrac{ \beta_t (M_Q+M_{\tilde Q})^2}{(1-\gamma)(\tilde \mu_t + \tilde \mu_{t-1})}
\le V^{\pi_{0}}(s) + \tfrac{(M_Q+M_{\tilde Q})^2}{2(1-\gamma) |\tilde \mu_d| },
\]
where the last inequality follows from $\tilde \mu_t = \tilde \mu_d \tsum_{j=0}^t \beta_j + \lambda_k \ge  |\tilde \mu_d| k(k+1)/2$.
It then follows from the selection of $\beta_t = t+1$ and the previous inequality that 
\begin{align}
\tsum_{t=0}^{k-1} \beta_t[V^{\pi_{t}}(s) - V^{\pi_{t+1}}(s)] 
&=  \tsum_{t=0}^{k-1} V^{\pi_t}(s) - kV^{\pi_k}(s) \nn \\
&\le k[V^{\pi_0}(s) + \tfrac{(M_Q+M_{\tilde Q})^2}{2(1-\gamma) |\tilde \mu_d| } - V^{\pi_k}(s)] \nn\\
&\le k[V^{\pi_0}(s) + \tfrac{(M_Q+M_{\tilde Q})^2}{2(1-\gamma) |\tilde \mu_d| } - V^{\pi^*}(s) ]. \label{eq:bnd_sum_pda_value}
\end{align}
Taking the telescopic sum of \eqnok{eq:spda_nonconvex},
and utilizing the above bound in \eqnok{eq:bnd_sum_pda_value}, we then have
\begin{align*}
- \tsum_{t=0}^{k-1} \beta_k \big[ \psi^{\pi_{k}}(s, \pi_{k+1}(s)) - \tfrac{\beta_k (M_Q+M_{\tilde Q})^2}{\tilde \mu_k + \tilde \mu_{k-1}}\big]
&\le \tsum_{t=0}^{k-1} \beta_t[V^{\pi_{t}}(s) - V^{\pi_{t+1}}(s)] +  \tsum_{t=0}^{k-1} \tfrac{ \beta_k^2 (M_Q+M_{\tilde Q})^2}{(1-\gamma)(\tilde \mu_k + \tilde \mu_{k-1})} \\
&\le k[V^{\pi_0}(s) + \tfrac{(M_Q+M_{\tilde Q})^2}{2(1-\gamma) |\tilde \mu_d| } - V^{\pi^*}(s) ]
+  \tsum_{t=0}^{k-1} \tfrac{ \beta_k^2 (M_Q+M_{\tilde Q})^2}{(1-\gamma)(\tilde \mu_k + \tilde \mu_{k-1})}\\
&\le k\big[V^{\pi_0}(s)  - V^{\pi^*}(s)  +  \tfrac{3(M_Q+M_{\tilde Q})^2}{2(1-\gamma) |\tilde \mu_d| }\big],
\end{align*}
where the last inequality follows from the fact that
$\tsum_{t=0}^{k-1} \tfrac{\beta_t^2}{\tilde \mu_t + \tilde \mu_{t-1}}
\le \tsum_{t=0}^{k-1} \tfrac{(t+1)^2}{|\tilde \mu_d| k(k+1)} \le \tfrac{k}{|\tilde \mu_d|}$.

Noticing the first inequality in \eqnok{eq:spda_nonconvex}, it then follows
from the above inequality by choosing $\bar k(s)$ as the iteration with the smallest 
value of $-\big[ \psi^{\pi_{k}}(s, \pi_{k+1}(s)) - \tfrac{\beta_k (M_Q+M_{\tilde Q})^2}{\tilde \mu_k + \tilde \mu_{k-1}}\big]$ that
\begin{align*}
- \tfrac{k(k+1)}{2} \big[ \psi^{\pi_{\bar k(s)}}(s, \pi_{\bar k(s)+1}(s)) - \tfrac{\beta_{\bar k(s)} (M_Q+M_{\tilde Q})^2}{\tilde \mu_{\bar k(s)} + \tilde \mu_{\bar k(s)-1}}\big]
\le k\big[V^{\pi_0}(s)  - V^{\pi^*}(s)  +  \tfrac{3(M_Q+M_{\tilde Q})^2}{2(1-\gamma) |\tilde \mu_d| }\big].
\end{align*}
The result in \eqnok{eq:PDA_nonconex_stoc1} then clearly follows from the above inequality and the fact that
$
0\le \tfrac{\beta_{\bar k(s)}}{\tilde \mu_{\bar k(s)} + \tilde \mu_{\bar k(s)-1}} \le  \tfrac{1}{|\tilde \mu_d|(k+1)}$.
The result in \eqnok{eq:PDA_nonconex_stoc2}
follows immediately from \eqnok{eq:PDA_nonconex_stoc1} and \eqnok{eq:bnd_pda_dist_stationary}.
\end{proof}

\vgap

We make a few remarks about the above results.
First,~\cref{the:PDA_nonconvex} appears to stronger than the corresponding one in~\cref{the:PMD_fun_stationary} for
PMD in the sense that it requires much weaker assumptions. 
Second, it is unclear if the result above also implies convergence to a stationary point; see the discussion after~\cref{the:pda_nonconvex}.
Third, while the number of iterations $k$ does not need to be fixed a priori when $\tilde \mu_d \ge 0$,~\cref{the:PDA_nonconvex} does require us to fix $k$ when $\tilde \mu_d < 0$, which helps reduce $\lambda_{k+1} -\lambda_k$. 
It will be interesting to see if more adaptive selections of $\lambda_k$ are possible for the PDA method when $\tilde \mu_d < 0$.

\section{Numerical Experiments} \label{sec:numerical_experiments}
We now compare the performance of policy dual averaging (PDA) against several algorithms from Stablebaselines3~\cite{raffin2021stable}, a state-of-the-art RL library.
We implemented PDA over policy mirror descent (PMD) for a couple reasons.
First, we only need to estimate the simpler action-value function rather than the augmented form in~\eqnok{eq:aug_value_function}.
Second, when the Q-function is convex, the PDA stepsize does not need to specify the number of iterations a priori like in PMD.

Let us now provide some implementation details.\footnote{Code available at \url{https://github.com/jucaleb4/RL-general-action-state}.}
We used a discount factor of $\gamma=0.99$. 
For function approximation, we considered (linear) random Fourier features~\cite{rahimi2007random} and neural networks. 
More details about the former can be found in~\cref{sec:appendix_pe}.
For the latter, we employed a fully connected neural network with 2 hidden layers, where each layer consists of an affine function of width 64 with tanh activation function, which is the same architecture used in Stablebaselines3~\cite{raffin2021stable}.
For policy evaluation, we estimated the value function by a simple Monte-Carlo approach while bootstrapping the remaining cost-to-go values.
Then to update the linear function approximation, we used \texttt{sklearn}'s stochastic gradient descent with the default settings and 100 epochs.
To update the neural network, we used the same update scheme as in Stablebaselines3.
\revise{Additionally, we tuned PDA's stepsize to improve the empirical performance. 
We performed a grid search over both the ``convex stepsize'' from~\cref{thm:4-1} and ``nonconvex stepsize'' from~\cref{the:PDA_nonconvex} as well as the stepsize magnitude $\lambda$ therein, and we chose the best one from the tuning phase. 
We also applied a grid search to tune the stepsize for other algorithms that we compared against in our experiments}.
Throughout this section, we refer to ``samples'' as a single timestep in the environment. 

\subsection{GridWorld with Traps}
We first consider a 2D GridWorld example over a 10x10 grid. 
An agent wants to reach a target by moving in one of the four cardinal directions each time step.
The agent suffers a unit cost for each step and an addition cost of 5 when they land on a trap spot; see also~\cite{dann2014policy}.
Once the agent reaches the target, no further cost is incurred.
The state space consists of the (x,y)-coordinates for both the agent and target, which are randomly chosen when the environment is initialized.
Thus, the state space is of size $\vert \mathcal S \vert = 10^4$. 
Ten randomly placed traps (dependent only on the environment seed) are set.
To avoid state enumeration, we use linear function approximation where the strongly convex function $\omega$  from~\eqnok{eq:omega_strong_convexity} is Shannon entropy. 
We also use a neural network where $\omega$ can either be Shannon entropy (so the Bregman distance in~\eqnok{eq:omega_strong_convexity} is the KL-divergence) or negative Tsallis entropy~\cite{tsallis1988possible,li2023policy}.

\cref{fig:gw} compares the various PDA algorithms against proximal policy optimization (PPO)~\cite{schulman2017proximal} and deep Q-network (DQN)~\cite{mnih2015human} as implemented in Stablebaselines3~\cite{raffin2021stable}.
We see the best performing PDA algorithm is with neural networks and Shannon entropy, labeled as \textit{pda-nn-kl}.
On average, it outperforms the closest competitor, PPO, after 50,000 samples.
Meanwhile, PDA with neural networks and Tsallis entropy (labeled \textit{pda-nn-tsallis}) eventually matches PPO's performance on average.
In contrast, both DQN (labeled \textit{dqn}) and PDA with linear function approximation (labeled \textit{pda-lin}) cannot learn a good policy.

\begin{figure}[t]
\begin{minipage}[t]{0.37\linewidth}
    \centering
    \includegraphics[width=\textwidth,valign=b]{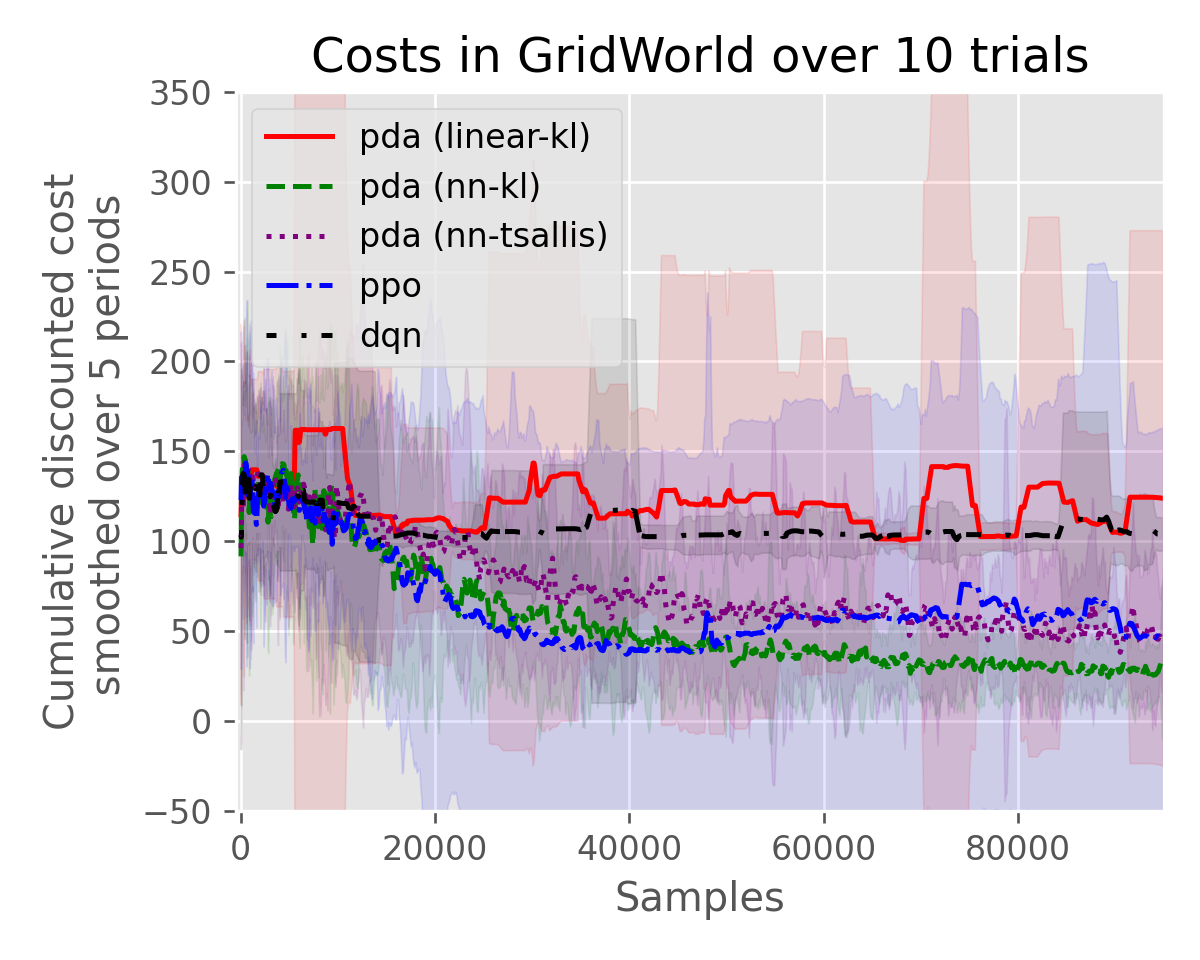}
    \captionof{figure}{Mean score and 95\% confidence interval (shaded region) in GridWorld.}
    \label{fig:gw}
\end{minipage}
\hspace{0.01\textwidth}
\begin{minipage}[t]{0.51\linewidth}
    \centering
    \includegraphics[width=\textwidth,valign=b]{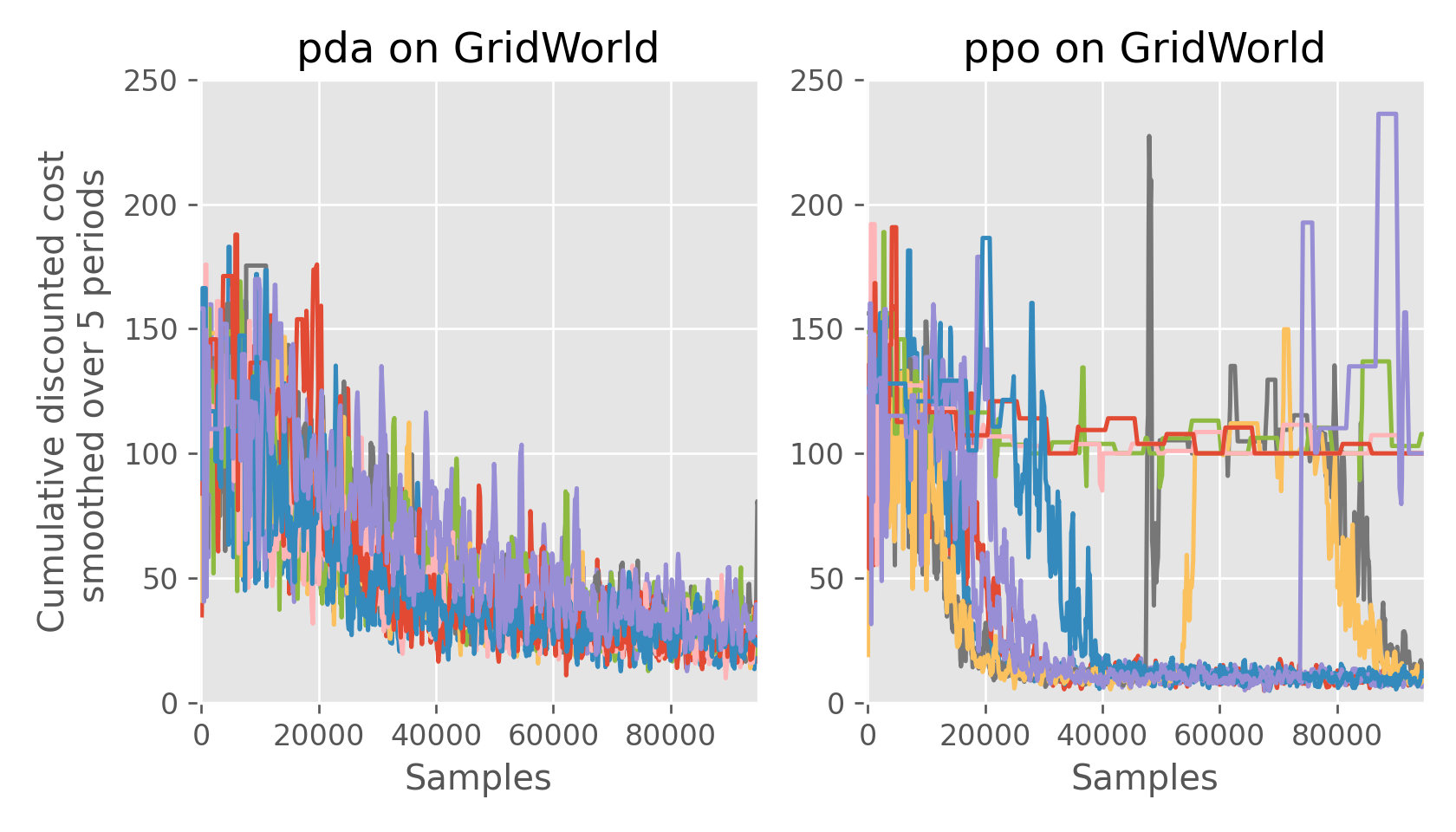}
    \captionof{figure}{Score for each of the ten seeds in GridWorld.}
    \label{fig:gw_h2h}
\end{minipage}
\vspace{-2em}
\end{figure}

Furthermore, PPO also has more variability in performance than PDA.
To better understand this phenomenon, we plot every seed's performance for \textit{pda-nn-kl} and \textit{ppo} in~\cref{fig:gw_h2h}.
Here, we can clearly see PDA has more consistent performance, where every seed obtains a relatively good policy. 
On the other hand, 4 out of 10 seeds in PPO fail to significantly improve after 100,000 samples.

\subsection{Lunar Lander}
Next, we consider the Lunar Lander problem, which is a simplified rocket trajectory optimization problem.
The goal is to land a rocket safely within a pre-defined landing spot using only four actions: fire either the left, right, or main (bottom) engine, otherwise not fire any engines. 
The rocket starts at the top center of the world with a random initial force applied to the center of mass. 
The agent receives various small costs for having the rocket be more stable, being closer to the landing site, and using the engine less often.
Meanwhile, a large negative cost is incurred for safely landing the rocket and a large positive cost for landing at the wrong site.\footnote{More details about the environment can be found at \url{https://gymnasium.farama.org}.}
It is important to note that we are considering a discounted cost.

In~\cref{fig:ll}, we compare the various RL algorithms.
This time, we see linear function approximation performs better in PDA than with neural networks.
Furthermore, PDA's cumulative discounted cost is higher than both PPO and DQN.
In our experience, we found PPO performs better than PDA when the cost is sparse, meaning a large portion of the cost arrives later in the sample trajectory after accomplishing a particular task.
This might be because PPO is designed to perform well on finite-horizon tasks, while PDA is designed for infinite-horizon discounted cost problems.

\begin{figure}[t]
\begin{minipage}[t]{0.33\linewidth}
    \centering
    \includegraphics[width=\textwidth,valign=b]{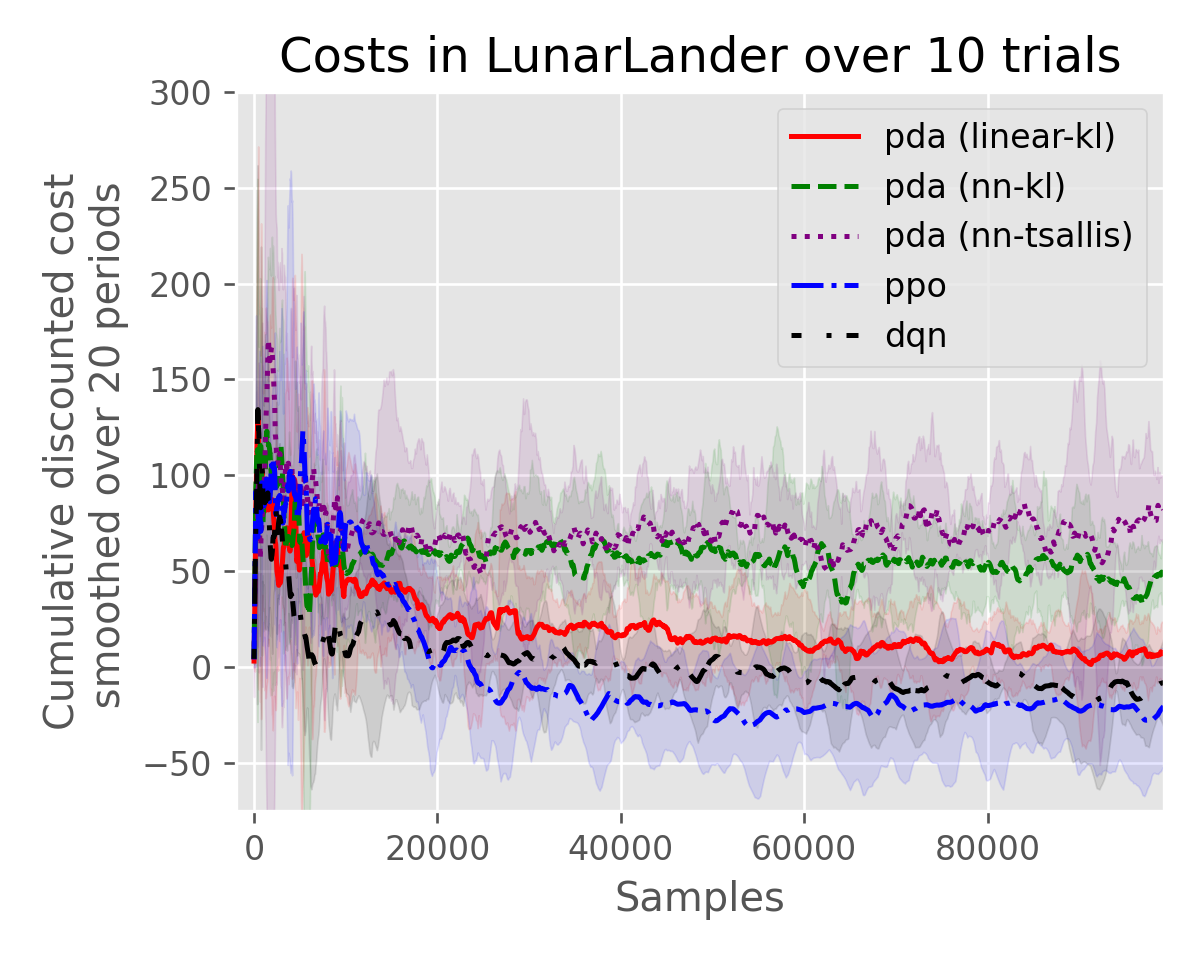}
    \captionof{figure}{Lunar Lander. See~\cref{fig:gw} for more details on the plot.}
    \label{fig:ll}
\end{minipage}
\hspace{0.01\textwidth}
\begin{minipage}[t]{0.3\linewidth}
    \centering
   \includegraphics[width=1\textwidth,valign=b]{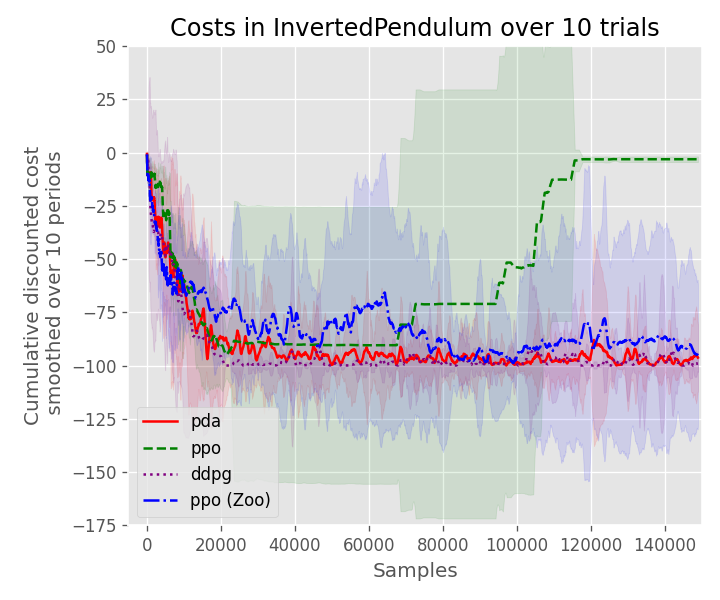}
    \captionof{figure}{Inverted pendulum.}
    \label{fig:invpend}
\end{minipage}
\begin{minipage}[t]{0.3\linewidth}
    \centering
    \includegraphics[width=1\textwidth,valign=b]{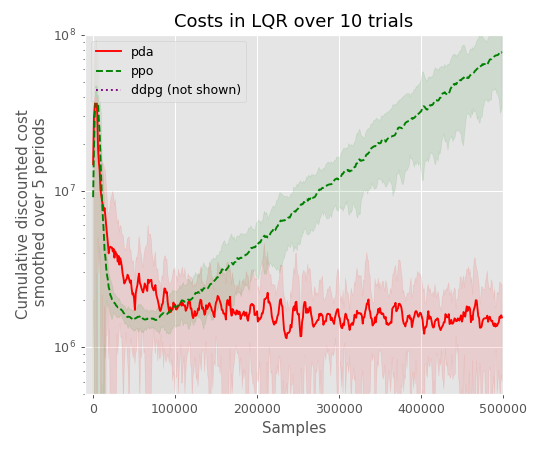}
    \captionof{figure}{Score on LQR. DDPG is not shown since its cost diverges.}
    \label{fig:lqr}
\end{minipage}
\vspace{-2em}
\end{figure}
\subsection{Inverted Pendulum}
Next, we consider Inverted Pendulum~\cite{barto1983neuronlike}.
In this continuous state and action space environment, the goal is to balance a pole fixed to a cart subject to gravitational and inertial forces (whose physics are simulated by MuJoCo).
The agent applies continuous forces (the action space is $\mathcal A = [-3,3]$) to the cart to balance the pole for as long as possible, where a negative cost of -1 is received for every time step the pole remains on the cart.
Clearly, the agent is incentivized to balance the pole for as long as possible.
The environment resets after 1,000 samples if the pole has not fallen already, so the optimal cumulative discounted cost is -100.

Before comparing the RL algorithms, let us discuss some implementation details in PDA.
Since PPO adds Gaussian noise to their policy (i.e., a randomized policy) to encourage action exploration, we similarly add independent zero-mean Gaussian noise whose covariance matrix is the identity matrix scaled by $1/{\sqrt{t}}$, where $t$ is current PDA iteration. 
This ensures the noise gradually decreases over time.
In fact, we found that without such exploration, the policy evaluation is not accurate enough to consistently improve the policy.
Also, the subproblem~\cref{eq:SPDA_step1} is solved using a recently developed universal smooth convex optimization solver~\cite{li2023simple}, which has good empirical performance and does not require knowledge of unknown Lipschitz smoothness parameters. 
In addition to running our tuned variants of PDA, PPO, and DDPG, we also ran PPO with pre-tuned hyperparameters from RL Baselines3 Zoo, labeled \textit{ppo (Zoo)}~\cite{raffin2021stable}. 

The results comparing PDA using neural networks and Stablebaselines3's implementation of PPO and deep deterministic policy gradient (DDPG) -- which is the analogue of DQN for continuous action space -- are shown in~\cref{fig:invpend}.
The plot showcases PDA and DDPG's superior performance over both variants of PPO.
In fact, the PPO that is (heuristically) tuned by our grid search diverges as the algorithm sees more samples.
This may be caused by overfitting its stepsize to the tuning phase, leading to poor performance when given more samples.

\subsection{Linear Quadratic Regulator for the Longitudinal Control of a Wide-Body Aircraft}
Our penultimate experiment is another continuous state and action space problem called linear quadratic regulator (LQR). 
In our instantiation of LQR, we consider a state-feedback control problem of the longitudinal dynamics of a Boeing 747 aircraft. 
Due to space constraints, we omit details, but the full model description can be found in~\cite{ju2022model}.
The main difference between Inverted Pendulum and LQR is the former has a bounded state-action space while the latter has an unbounded space of~$\mathbb{R}^{\nS + \nA}$, where $\nS = 5$ and $\nA = 4$.
Thus, this environment explores how different algorithms choose actions when the state can diverge.
Similar to the Inverted Pendulum experiments, we also applied action exploration to PDA.

We again compare PDA with neural networks against PPO and DDPG, as shown in~\cref{fig:lqr}.
PDA exhibits the best long-term performance, while PPO seems to have good initial performance, but it eventually diverges. 
Meanwhile, DDPG is unable to address the unboundedness as seen by its unbounded cost.

\subsection{Humanoid}
\begin{figure}[t]
\begin{minipage}[t]{0.38\linewidth}
    \centering
    \includegraphics[width=\textwidth,valign=b]{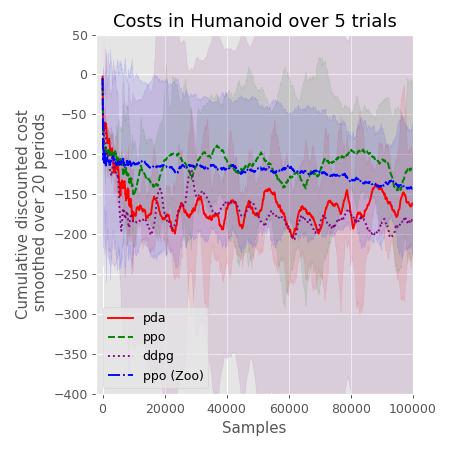}
    \captionof{figure}{Humanoid. See~\cref{fig:gw} for more details on the plot.}
    \label{fig:humanoid}
\end{minipage}
\hspace{0.01\textwidth}
\begin{minipage}[t]{0.60\linewidth}
    \centering
    \includegraphics[width=\textwidth,valign=b]{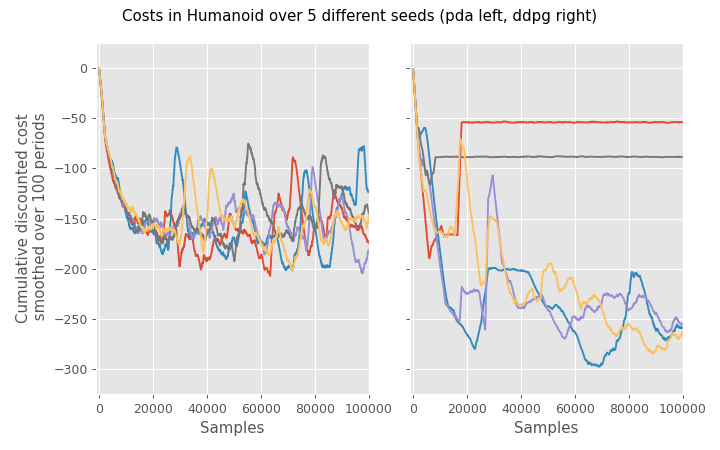}
    \captionof{figure}{Score for each of the five seeds in Humanoid.}
    \label{fig:humanoid_seed}
\end{minipage}
\vspace{-2em}
\end{figure}

We end with a high-dimensional continuous state and action space problem. 
In the Humanoid control problem, the goal is to control a 3D bipedal robot mimicking human movements. 
With a state space of $\mathbb{R}^{348}$ and action space of $[-0.4,0.4]^{17}$, the control problem aims to have the humanoid walk as fast as possible without falling over. 
More specifically, the humanoid receives positive rewards (equivalently, negative costs) for staying upright and moving forward while facing negative rewards (equivalently, positive costs) for large or aggressive movements. 
We refer to~\cite{tassa2012synthesis} for full details of the control problem.
We again applied action exploration to PDA and also ran PPO tuned by RL Baselines3 Zoo.

The results are shown in~\cref{fig:humanoid}. We see PDA on average outperforms both tuned variants of PPO. Meanwhile, PDA on average performs similarly to DDPG in the initial phase, and then DDPG on average improves after 80,000 samples. However, the performance of DDPG has larger variance. To better elucidate this, we plotted each seed's performance in~\cref{fig:humanoid_seed}. Analogous to performance for GridWorld in~\cref{fig:gw_h2h}, we see PDA has more consistent performance across different seeds, while DDPG varies more. In fact, 2 out of the 5 seeds for DDPG converge to a policy with significantly worse performance than from PDA. This sub-optimal performance is likely due to the fact DDPG, which employs function approximation with greedy policy updates, lacks convergence guarantees, while we already established the convergence of PDA.

\section{Concluding Remarks} \label{sec:conclusion}
In this paper, we study new policy optimization methods
for solving Markov decision process and reinforcement learning problems
over general state and action spaces. 
We first present a generalization of the policy mirror descent method
and suggest a novel function approximation scheme into it to deal with general state spaces.
We then introduce a new class of policy dual averaging (PDA) methods for reinforcement learning
that might be more amenable for function approximation in that it only requires
the approximation of the action-value function.
Numerical experiments show PDA exhibits robust and consistent performance across a range of optimal control and planning problems, sometimes outperforming state-of-the-art RL algorithms. 
We support these promising numerical performances with new convergence properties for both methods. 
\revise{Suppose certain curvature conditions are satisfied as described by~\eqnok{eq:def_convexity_PDA} and the surrounding text. In the deterministic case with exact action-value function evaluation, we establish linear convergence to global optimality. 
In the stochastic setting with noisy estimates of the action-value function, we derive a sublinear $O(k^{-1/2})$ (resp.~$O(k^{-1})$) rate of convergence to the neighborhood of global optimality with a general convex (resp.~strongly convex) regularizer, where the neighborhood size is proportional to the bias of the estimator.
On the other hand, when the curvature conditions are violated, both methods exhibit $O(k^{-1})$ convergence rates as defined by~\eqnok{eq:stationary_PMD0} and~\eqnok{eq:stationary_PMD1}. For PMD only, these results may imply convergence to a stationary point in the sense of~\eqnok{eq:stationary_pt}.}

There appears to be connections between such stationary points and  gradient-based local optimality conditions (see the discussion following~\cref{lem:monotonicity}). We leave a detailed investigation of this connection, potentially important for understanding the convergence of reinforcement learning algorithms, for future work.
\revise{Finally, an important future direction is to generalize the policy evaluation approach from~\cref{sec:generalization} to learn from online data, which may have time dependency and thus lack i.i.d.~samples.}

\vgap

\noindent {\bf Acknowledgment:}
The authors sincerely thank the reviewers and Professor Andrzej Ruszczynski for their thoughtful and constructive suggestions, which have significantly improved the quality of this manuscript.

\appendix
\section{Policy Evaluation} \label{sec:appendix_pe}
In this section, we explore when policy evaluation can be accurate.
For notational convenience, let $\cZ := \cS \times \cA$ be the state-action pair space.
Throughout the appendix, let $\|\cdot\|_\cZ$ be the uniform norm. 
\begin{definition}[Continuous stochastic control MDP]
  Consider an MDP $\mathcal M = (\cS, \cA, \cP, c, \gamma)$ with regularizer $h$, where both $c(\cdot,\cdot): \cZ \to [0,1]$ and $h^\cdot(\cdot) : \cZ \to \mathbb{R}$ are continuous in $\cZ$.
  In addition, suppose the transition kernel is defined according to the transfer function $s' = f(s,a,\xi)$, where $\xi$ is an i.i.d.~random variable (i.e., $\cP(s' \vert s,a) = \cP\{\xi \vert s' = f(s,a,\xi)\}$) and $f(\cdot,\cdot,\tilde{\xi})$ is continuous in $\cZ$ for any fixed $\tilde{\xi}$.
\end{definition}

The continuous stochastic control MDP is a special class of the stochastic control environment~\cite{bertsekas1996stochastic}, and it captures stochastic dynamical systems.
Our goal is to show this class of MDPs has action-value functions $Q^{\pi}$ that can be approximated arbitrarily well.
To do so, we will utilize universal kernels, which are dense in the space of continuous functions over a compact set~\cite{micchelli2006universal}.
To start, we will show the action-value $Q^\pi$ (defined in~\cref{sec_prob_formulation}) is continuous, which allows it to be approximated arbitrarily well.
\begin{lemma} \label{lem:continuous_Q}
  Given a continuous stochastic control MPD and a policy $\pi : \cS \to \cA$ that is continuous in $\cS$, then $Q^\pi$ is continuous in $\cZ$.
\end{lemma}
\begin{proof}
  We sketch the proof.
  Define the $T$-truncated action-value function,
  \begin{align*}
    Q^\pi_T(s,a) :=
    \bbe\left[\tsum_{t=0}^{T-1} \gamma^t [c(s_t, a_t) + h^{a_t}(s_t)] \mid
    s_0 = s, a_0 = a, a_t = \pi(s_t),
    s_{t+1} \sim \cP(\cdot | s_t, a_t)\right],
  \end{align*}
  By the properties of a continuous stochastic control MDP and continuity of $\pi$, a proof by induction shows that $Q^\pi_T$ is continuous on $\cZ$ for all $T \geq 1$.
  Since $\|Q^\pi_T\|_{\cZ}$ is bounded by $\bar{c}/(1-\gamma)$ and $\gamma < 1$, then $Q^\pi_T$ uniformly converges to $Q^\pi$. 
  By the uniform limit theorem, $Q^\pi$ is also continuous in $\cZ$.
\end{proof}

The policy $\pi$ will be set to $\pi_k$ from PDA, which is the solution to an optimization subproblem of~\eqnok{eq:SPDA_step1}.
We assume a solution to the subproblem exists.
We will show $\pi_k$ is continuous, which in view of the previous lemma, verifies $Q^{\pi_k}$ is continuous as well.
A similar result can be shown for PMD from~\eqnok{eq:PMD_step_fa}.
Before stating the result, we say the action set $\cA$ is \textit{polyhedral} when $\mathcal A := \{x \in \mathbb{R}^{\nA} : Ax=b, x \geq \mathbf 0\}$ for some full-rank matrix $A$ and vector $b$.
A polyhedral set $\cA$ generalizes both the finite-action and certain general action space problems.
For the following result, we skip its proof since it is a well-known result from sensitivity analysis; see~\cite[Theorem 5.1]{fiacco1990sensitivity} and references therein.
\begin{corollary} \label{cor:continuous_policy}
  \sloppy If $\cA$ is polyhedral, the function approximation $\tilde Q(\cdot, \cdot; \theta_k)$ and its derivative are continuously differentiable for all states $s$ and weights $\theta_k$, and $\pi_{k+1}(s)$ satisfies the strict complementary slackness condition for any state $s$, 
  then $\pi_{k+1}$ is locally continuously differentiable.
\end{corollary}

The assumption on strict complementary slackness seems to be standard in the sensitivity analysis of optimization problems~\cite{fiacco1990sensitivity,bonnans2013perturbation}.
Moreover, it is satisfied in the finite-action case with $D(\cdot,\cdot)$ in~\eqnok{eq:SPDA_step1} set to the KL divergence, which ensures $\pi_{k+1}(s)$ is strictly positive.

\subsection{Generalization Error Bounds} \label{sec:generalization}
Denote $z \equiv [s,a] \in \cZ$ as a state-action pair.
With some abuse of notation, we may write $Q^\pi(z)$ in place of $Q^\pi(s,a)$.
Throughout the section, we make the following assumption about our Markov chain model.
\begin{assumption} \label{asmp:compact_and_ergodic}
  The state-action space $\cZ$ is compact. 
  We also assume a resetting model, where the Markov chain $\{z_t\}_{t \geq 0}$ can be reset to any arbitrary state-action pair $z$ at any time.
\end{assumption}

By the Markov property, resetting allows us to obtain independent trajectories. 
It may be possible to extend our results to the more difficult online (i.e., single trajectory) setting by making additional ergodicity assumptions on the Markov chain.
However, to the best of our knowledge, policy evaluation generalization bounds for general state-action spaces seems to be open even in the resetting model, so our goal is to provide some intuition about this simpler setting here.

Let $\hat{\cA}$ be a finite subset of $\cA$. Later, we will address how to extend our results to the more general $\cA$.
Consider the linear function approximation, 
\begin{align} \label{eq:random_features_approx}
  \tilde{Q}(z;\theta) 
  := 
  \sum\limits_{i=1}^N \theta_i \kappa(z_i, z), \quad \text{where} \quad \kappa(z_i,z) := \frac{1}{D} \sum\limits_{j=1}^D \phi_j(z_i)^T\phi_j (z),
\end{align}
while $\theta \in \mathbb{R}^N$ are the weights to be fitted/learned, and $\{\phi_j: \cZ \to \mathbb{R}^d\}_{j=1}^D$ and $\{z_i\}_{i=1}^N$ are some to-be-determined feature maps and basis points, respectively.
By choosing random Fourier features for $\phi_j$~\cite{rahimi2007random}, 
we ensure $\tilde{Q}(\cdot; \theta)$ is infinitely differentiable,
the gradient $\nabla_a \tilde Q(s,a;\theta)$ is bounded (and hence $\tilde Q(s, \cdot; \theta)$ is Lipschitz continuous) with high probability, 
and $\tilde{Q}(\cdot; \theta)$ approximates a Gaussian kernel, which is a \textit{universal kernel}, i.e., forms a space that is dense in space of continuous functions over a compact domain~\cite{micchelli2006universal}. 

We will discuss how to best fit/learn $\theta$ to the current policy $\pi_k$. 
First, let us fix any error $\varepsilon > 0$ and failure rate $\delta \in (0,1)$.
Next, sample $z_i$ independently and uniformly from $\mathcal S \times \hat{\mathcal A}$.
For each $z_i$, pair it with an independent sample $y_i \sim Y(z_i)$, where $Y(z)$ is some distribution of ``high-quality'' estimates of $Q^{\pi_k}(z)$.
To be more precise, we can sample such a $y_i$ by first gathering $m > 0$ estimates $\{\tilde{y}_i^{(j)}\}_{j=1}^m$, where $\tilde{y}_i^{(j)}$ is an independent and nearly unbiased estimate of $Q^{\pi_k}(z_i)$ with absolute error at most $2\bar{c}/(1-\gamma)$. 
Such a $\tilde{y}_i^{(j)}$ can be constructed with the resetting model and a Monte Carlo estimate~\cite{li2023policy,RichardBarto2018}.
Then, aggregate them into $y_i = m^{-1}\sum_{j=1}^m \tilde{y}_i^{(j)}$. 
By construction, $y_i$ has absolute error at most $2\bar{c}/(1-\gamma)$ with probability 1. 
And by the Azuma-Hoeffding inequality, the absolute error can be improved to be at most $\varepsilon$ with probability $1-\delta$ when we set $m = O\{[\bar{c}/(\varepsilon \cdot (1-\gamma))]^2 \ln(1/\delta)\}$.
In particular, by setting $\delta = [\varepsilon \cdot (1-\gamma)/(2\bar{c})]^2$, then
\begin{align} \label{eq:small_err}
    \mathbb{E}_{z \sim F_\cZ, y \sim Y(z)} (Q^{\pi_k}(z) - y)^2
    \leq
    \varepsilon^2 \cdot (1-\delta) + (2\bar{c}/(1-\gamma))^2 \cdot \delta
    \leq
    2\varepsilon^2,
\end{align}
where $F_\cZ$ is any distribution over $\cZ$.
\revise{Then compute the best fitting weights $\theta_k$ from~\eqnok{eq:random_features_approx} by solving a ridge regression problem $(K + \lambda I)\theta = y$, where the data-label pairs~$\xi_k := \{(z_i,y_i)\}_{i=1}^N$ appear in the kernel matrix $K = \{\kappa(z_i,z_j)\}_{i,j}$ from~\eqnok{eq:random_features_approx} and labels $y$.
Let us briefly derive the computational cost to solve the ridge regression problem.
Since the regularization parameter is $\lambda = N^{-1/2}$ based on~\cite{rudi2017generalization}, and the random Fourier features $\phi_j(\cdot)$ in~\eqnok{eq:random_features_approx} are bounded by $\sqrt{2}$ elementwise~\cite{rahimi2007random}, then one can verify that the smallest and largest eigenvalue satisfy, respectively, $N^{-1/2} \leq \lambda_{N}(K + \lambda I)$ and $\lambda_1(K + \lambda I) \leq 3N$. 
From this, we deduce the iteration complexity for solving the ridge regression problem to error $\epsilon$ with conjugate gradient or the accelerated gradient method is $O(N^{-3/4}\log(1/\epsilon))$~\cite{nesterov2013introductory}.}

We are now ready to bound the generalization error of $\tilde{Q}(\cdot; \theta_k)$.
Recall for any probability measures $\rho$ and $\rho'$ over a measurable space where $\rho$ is absolutely continuous w.r.t~$\rho'$, the Radon-Nikodym derivative is $\frac{d\rho}{d\rho'}$. 
The big-O notation $\tilde{O}(\cdot)$ hides polylogarithmic dependence.
\begin{proposition} \label{prop:error_bounds}
  Let $\rho^*$ be a probability measure over $\cS$.
  If ${Q}^{\pi_k}(\cdot)$ is continuous and bounded, $m = \tilde{O}\{[\bar{c}/(\varepsilon \cdot (1-\gamma))]^2\}$, and $N$ is sufficiently large,
  then there exists a $D = \tilde{O}(\sqrt{N})$ and universal positive constant $c_1$ such that
  \begin{align*}
    \mathbb{E}_{\xi_k} \mathbb{E}_{s \sim \rho^*}[ \max_{a \in \hat{\cA}} (\tilde{Q}(s,a;\theta_k) - {Q}^{\pi_k}(s,a))^2] 
    \leq
    \textstyle 
    16 \big\| \frac{d\rho^*}{dU(\cS)} \big\|_\infty \vert \hat{\cA \vert} (\frac{5c_1}{\sqrt{N}} + \varepsilon^2),
  \end{align*}
  where $D$ and $N$ are from~\eqnok{eq:random_features_approx} and $dU(\cS)$ is the uniform density function over $\cS$.
\end{proposition}
\begin{proof}
  Let $\cH$ be a reproducing kernel Hilbert space (mapping $\cZ$ to reals) defined as the completion of the linear span of a Gaussian (i.e., universal) kernel.
  Define 
  $\mathcal{E}(f) :=  \mathbb{E}_{\xi_k} \mathbb{E}_{z \sim U(\cS \times \hat{\cA}), y \sim Y(z)} (f(z) - y)^2$, where $U(\cS \times \hat{\cA})$ is the uniform distribution over $\cS \times \hat{\cA}$.
  
  Since $\mathcal H$ is defined as a universal kernel over a compact space $\cZ$ and $Q^{\pi_k}$ is continuous and bounded, then for any $\varepsilon > 0$, there exists an $\hat{f} \in \mathcal H$ s.t.~$\|\hat f - Q^{\pi_k}\|_{\cZ} \leq \varepsilon$~\cite{micchelli2006universal}.
  Combining this observation with~\eqnok{eq:small_err}, we derive
      $\mathcal{E}(\hat f) 
      \leq 
      2[\|\hat f - Q^{\pi_k}\|^2_\cZ + \mathcal{E}(Q^{\pi_k})]
      \leq
      2(\varepsilon^2 + 2\varepsilon^2)$.
  Thus, we can conclude
  \begin{align*}
    \mathbb{E}_{\xi_k} \mathbb{E}_{s \sim \rho^*}[ \max_{a \in \hat{\cA}} (\tilde{Q}(s,a;\theta_k) - {Q}^{\pi_k}(s,a))^2] 
    &\leq
    \textstyle \| \frac{d\rho^*}{dU(\cS)} \|_\infty \vert \hat{\cA} \vert \cdot 
    \mathbb{E}_{\xi_k} \mathbb{E}_{z \sim U(\cS \times \hat{\cA})} (\tilde{Q}(z;\theta_k) - {Q}^{\pi_k}(z) )^2
    \\
    &\leq
    2\textstyle \| \frac{d\rho^*}{dU(\cS)} \|_\infty \vert \hat{\cA} \vert \cdot [  \mathcal{E}(\tilde{Q}(\cdot; \theta_k)) +  \mathcal{E}(Q^{\pi_k})]
    \\
    &\leq
    \textstyle 2\| \frac{d\rho^*}{dU(\cS)} \|_\infty \vert \hat{\cA} \vert \cdot [ {40c_1}/{\sqrt{N}} + \inf_{f \in \cH} \mathcal{E}(f) + 2\varepsilon^2 ] \\
    &\leq
    \textstyle {2 \| \frac{d\rho^*}{dU(\cS)} \|_\infty \vert \hat{\cA} \vert } \cdot [40c_1/\sqrt{N} + 8\varepsilon^2],
  \end{align*}
  where the first line is by a change-in-measure argument and compactness of $\cS$, the third line is by~\cite[Corollary 1]{rudi2017generalization} as well as~\eqnok{eq:small_err}, and the last line is with the help of our earlier bound on $\mathcal{E}(\hat f)$ and feasibility of $\hat f$.
\end{proof}

More details about the minimum size of $N$ are in~\cite[Corollary 1]{rudi2017generalization}.
The assumption on $Q^{\pi_k}(\cdot)$ is satisfied based on our arguments in the previous section for the continuous stochastic control MDP.
If the action space is finite (i.e., $Q^{\pi_k}(s,\cdot)$ has a finite domain for any state $s \in \cS$), we can extend it to be continuous by working in the space of randomized policies like in~\eqnok{eq:PMD_step_finite} and the surrounding discussions.
While the above result holds over a finite subset $\hat{\cA} \subset \cA$, 
it can be extended to general action spaces $\cA$ if we further assume both $Q^{\pi_k}(s,\cdot)$ and $\tilde{Q}^{\pi_k}(s,a)$ are Lipschitz continuous w.r.t.~the actions (i.e.~\cref{eq:Lip_psi} and~\cref{eq:Lip_Q}), fix $\hat{\cA}$ to be an $\epsilon$-net of $\cA$, and apply Monte Carlo sampling-type arguments~\cite[Section 5.3]{shapiro2021lectures}. 
Setting $\rho^* = \nu^*$, we can bound both deterministic error terms~\eqnok{eq:pda_finite_assum1} and~\eqnok{eq:bnd_error_spda}.
More specifically, we obtain $N = O(\epsilon^{-4})$ i.i.d.~samples of $z_i$, each of which requires $m = \tilde{O}(\epsilon^{-2})$ samples to reduce the variance, to ensure the mean squared error is at most $\epsilon^2$. 
One can similarly bound~\eqnok{eq:pda_finite_assum3}.
By Jensen's inequality, the aforementioned deterministic error terms are then at most $\epsilon$.

\renewcommand\refname{Reference}

\bibliographystyle{plain}
\bibliography{GeorgeLan}

\end{document}